\documentclass[lettersize,journal]{IEEEtran}
\usepackage{amsmath,amsfonts}
\usepackage{array}
\usepackage{textcomp}
\usepackage{stfloats}
\usepackage{url}
\usepackage{verbatim}
\usepackage{cite}
\usepackage{mathrsfs}
\usepackage{comment}
\usepackage{booktabs}       
\usepackage{amsmath}
\usepackage{amssymb}
\usepackage{xargs}
\usepackage{mathtools} 

\usepackage{longtable}
\usepackage{algorithm} 
\usepackage{algorithmic}

\usepackage{dsfont}
\usepackage{multirow}  
\usepackage{makecell}  

\usepackage{graphicx}
\usepackage{subcaption}

\newcommand{\cB}{\mathcal{B}}

\newcommand{\cD}{\mathcal{D}}

\newcommand{\cM}{\mathcal{M}}

\newcommand{\cO}{\mathcal{O}}

\newcommand{\cW}{\mathcal{W}}
\newcommand{\cX}{\mathcal{X}}

\newcommand{\EE}{\mathbb{E}}

\newcommand{\II}{\mathbb{I}}

\newcommand{\NN}{\mathbb{N}}

\newcommand{\argmin}{\mathop{\mathrm{argmin}}}
\newcommand{\argmax}{\mathop{\mathrm{argmax}}}

\newcommand{\ind}{\mathds{1}}  

\newcommand{\rbr}[1]{\left(#1\right)}
\newcommand{\sbr}[1]{\left[#1\right]}

\newcommand{\abr}[1]{\left|#1\right|}

\newcommand{\BlackBox}{\rule{1.5ex}{1.5ex}}  

\newenvironment{proof}{\par\noindent{\bf Proof\ }}{\hfill\BlackBox\\[2mm]}

\newtheorem{theorem}{Theorem}

\newtheorem{lemma}[theorem]{Lemma}

\newtheorem{corollary}[theorem]{Corollary}
\newtheorem{definition}[theorem]{Definition}

\newtheorem{assumption}[theorem]{Assumption}

\allowdisplaybreaks[4]
\newcommand{\statespace}[0]{\mathcal{X}} 
\newcommand{\statesize}[0]{X} 
\newcommand{\state}[0]{x} 
\newcommand{\absorbstate}[0]{x^\dag} 
\newcommand{\actionspace}[0]{\mathcal{A}} 
\newcommand{\actionsize}[0]{A} 
\newcommand{\action}[0]{a} 

\newcommand{\transeasy}[0]{P} 
\newcommand{\transprieasy}[0]{\widetilde{P}} 
\newcommand{\episode}[0]{k} 
\newcommand{\episodetotal}[0]{K} 
\newcommand{\horizon}[0]{h} 
\newcommand{\horizontotal}[0]{H} 
\newcommand{\rewdis}[0]{\mathcal{R}} 
\newcommand{\reward}[0]{r} 

\newcommand{\rewprieasy}[0]{\widetilde{r}} 
\newcommand{\rewcumeasy}[0]{R} 
\newcommand{\rewcumprieasy}[0]{\widetilde{R}} 
\newcommand{\policy}[0]{\pi} 
\newcommand{\valuef}[0]{V} 
\newcommand{\qvaluef}[0]{Q} 
\newcommand{\regret}[0]{\text{Regret}(K)} 

\newcommand{\visitxatotaleasy}[0]{N} 
\newcommand{\visitxaxtotaleasy}[0]{N} 
\newcommand{\visitxatotalbineasy}[0]{\ddot{N}} 

\newcommand{\visitxaxtotalbineasy}[0]{\ddot{N}}
\newcommand{\visitxatotalopteasy}[0]{\bar{N}} 

\newcommand{\visitxaxtotalopteasy}[0]{\bar{N}}
\newcommand{\visitxatotalprieasy}[0]{\widetilde{N}} 
\newcommandx{\visitxaxtotalpri}[0]{\widetilde{N}_h^b(x,a,x^{\prime})}
\newcommandx{\visitxaxtotalprieasy}[0]{\widetilde{N}}

\newcommand{\visitxatotalcrupri}[0]{\widetilde{N}_h^{\text{cru},b}(x,a)} 
\newcommandx{\visitxaxtotalcrupri}[0]{\widetilde{N}_h^{\text{cru},b}(x,a,x^{\prime})}
\newcommandx{\visitxaxtotalcruprieasy}[0]{\widetilde{N}}

\newcommand{\visitxatotalrefpri}[0]{\widetilde{N}_h^{\text{ref},b}(x,a)} 

\newcommandx{\visitxaxtotalrefprieasy}[0]{\widetilde{N}}
\newcommand{\visitxatotalref}[0]{N_h^{\text{ref},b}(x,a)} 

\newcommand{\policyset}[0]{\phi} 
\newcommand{\elimpolicyset}[0]{\psi} 
\newcommand{\accuracy}[0]{\xi} 
\newcommand{\badtrajs}[0]{\cB} 

\newcommand{\infreqtuples}[0]{\cW} 
\newcommand{\abtrans}[0]{P'} 
\newcommand{\crutrans}[0]{\widetilde{P}^{\text{cru},\batchindx}} 
\newcommand{\reftrans}[0]{\widetilde{P}^{\text{ref},\batchindx}} 

\newcommand{\pripara}[0]{\varepsilon} 
\newcommand{\priconf}[0]{\beta} 

\newcommand{\confrewf}[0]{E_{\varepsilon,\delta}} 
\newcommand{\confcountxa}[0]{E_{\varepsilon,\delta}} 
\newcommand{\confcountxax}[0]{E_{\varepsilon,\delta}} 

\newcommand{\protocol}[0]{\mathcal{T}} 
\newcommand{\randomizer}[0]{\mathcal{E}} 
\newcommand{\shuffler}[0]{\mathcal{F}} 
\newcommand{\analyzer}[0]{\mathcal{G}} 
\newcommand{\batchdata}[0]{D} 
\newcommand{\batchdataspace}[0]{\mathcal{D}}
\newcommand{\batchnumber}[0]{B}
\newcommand{\batchindx}[0]{b}
\newcommand{\batchlength}[0]{L}

\newcommand{\userspace}[0]{\mathcal{U}} 
\newcommand{\user}[0]{u} 
\newcommand{\traj}[0]{S} 
\newcommand{\trajset}[0]{\mathcal{S}} 

\newcommand{\prob}[0]{\mathbb{P}} 
\newcommand{\expect}[0]{\mathbb{E}} 

\usepackage[colorlinks=true, linkcolor=blue, citecolor=blue, urlcolor=blue]{hyperref}

\usepackage{xcolor} 

\begin{document}

\title{Near-Optimal Reinforcement Learning with Shuffle Differential Privacy
\thanks{
\textcolor{blue}{This work has been submitted to the IEEE for possible publication. Copyright may be transferred without notice, after which this version may no longer be accessible.}
}}

\author{Shaojie Bai, 
Mohammad Sadegh Talebi, 
Chengcheng Zhao,~\IEEEmembership{Member,~IEEE,} \\
Peng Cheng,~\IEEEmembership{Member,~IEEE,} 
and Jiming Chen, \IEEEmembership{Fellow,~IEEE} 
\IEEEcompsocitemizethanks{\IEEEcompsocthanksitem Shaojie Bai, Chengcheng Zhao, Peng Cheng, Jiming Chen are with College of Control Science and Engineering, Zhejiang University, 310027 Hangzhou, China. (email: bai\_shaojie@zju.edu.cn, chengchengzhao@zju.edu.cn, lunarheart@zju.edu.cn, cjm@zju.edu.cn).%
\IEEEcompsocthanksitem Mohammad Sadegh Talebi is with Department of Computer Science, University of Copenhagen, 2100, Copenhagen, Denmark. (email: m.shahi@di.ku.dk).}%
}


\markboth{Journal of \LaTeX\ Class Files,~Vol.~14, No.~8, August~2021}%
{Shell \MakeLowercase{\textit{et al.}}: A Sample Article Using IEEEtran.cls for IEEE Journals}


\maketitle

\begin{abstract}
Reinforcement learning (RL) is a powerful tool for sequential decision-making, but its application is often hindered by privacy concerns arising from its interaction data. 
This challenge is particularly acute in advanced networked systems, where learning from operational and user data can expose systems to privacy inference attacks. 
Existing differential privacy (DP) models for RL are often inadequate: the centralized model requires a fully trusted server, creating a single point of failure risk, while the local model incurs significant performance degradation that is unsuitable for many networked applications. 
This paper addresses this gap by leveraging the emerging shuffle model of privacy, an intermediate trust model that provides strong privacy guarantees without a centralized trust assumption. 
We present Shuffle Differentially Private Policy Elimination (SDP-PE), the first generic policy elimination-based algorithm for episodic RL under the shuffle model. 
Our method introduces a novel exponential batching schedule and a ``forgetting'' mechanism to balance the competing demands of privacy and learning performance. 
Our analysis shows that SDP-PE achieves a near-optimal regret bound, demonstrating a superior privacy-regret trade-off with utility comparable to the centralized model while significantly outperforming the local model. 
The numerical experiments also corroborate our theoretical results and demonstrate the effectiveness of SDP-PE.
This work establishes the viability of the shuffle model for secure data-driven decision-making in networked systems.
\end{abstract}

\begin{IEEEkeywords}
Reinforcement learning, online learning, networked system, shuffle differential privacy, regret  
\end{IEEEkeywords}
\section{Introduction}
\IEEEPARstart {R}{einforcement} learning (RL) has gained remarkable attraction to optimize networked systems~\cite{8714026}, 
such as 
wireless network,
smart grid, 
and autonomous vehicle networks. 
A common operational paradigm is online reinforcement learning, as illustrated in Figure~\ref{fig: networkRL case}: a central agent (e.g., a network coordinator) performing "centralized learning" to optimize a global policy, while providing streaming services to a sequence of online nodes or user requests. 
For each request, the central agent deploys its global policy to the corresponding distributed node (e.g., 5G small cell or cognitive radio node). 
The node locally and autonomously executes this policy to generate the interaction trajectory (e.g., an $H$-steps communication session), and the trajectory is then sent back to the central agent for updating the global policy.
For instance, in Dynamic Spectrum Access (DSA), a cognitive radio node runs the policy locally, sensing its local spectrum state to choose a channel. 
The coordinator's goal is to learn a global operational policy that optimizes network-wide performance (e.g., maximizing spectral efficiency) using the interaction data fed back by the nodes. 
\begin{figure}[!htbp]
  \centering
  \includegraphics[width=0.4\textwidth]{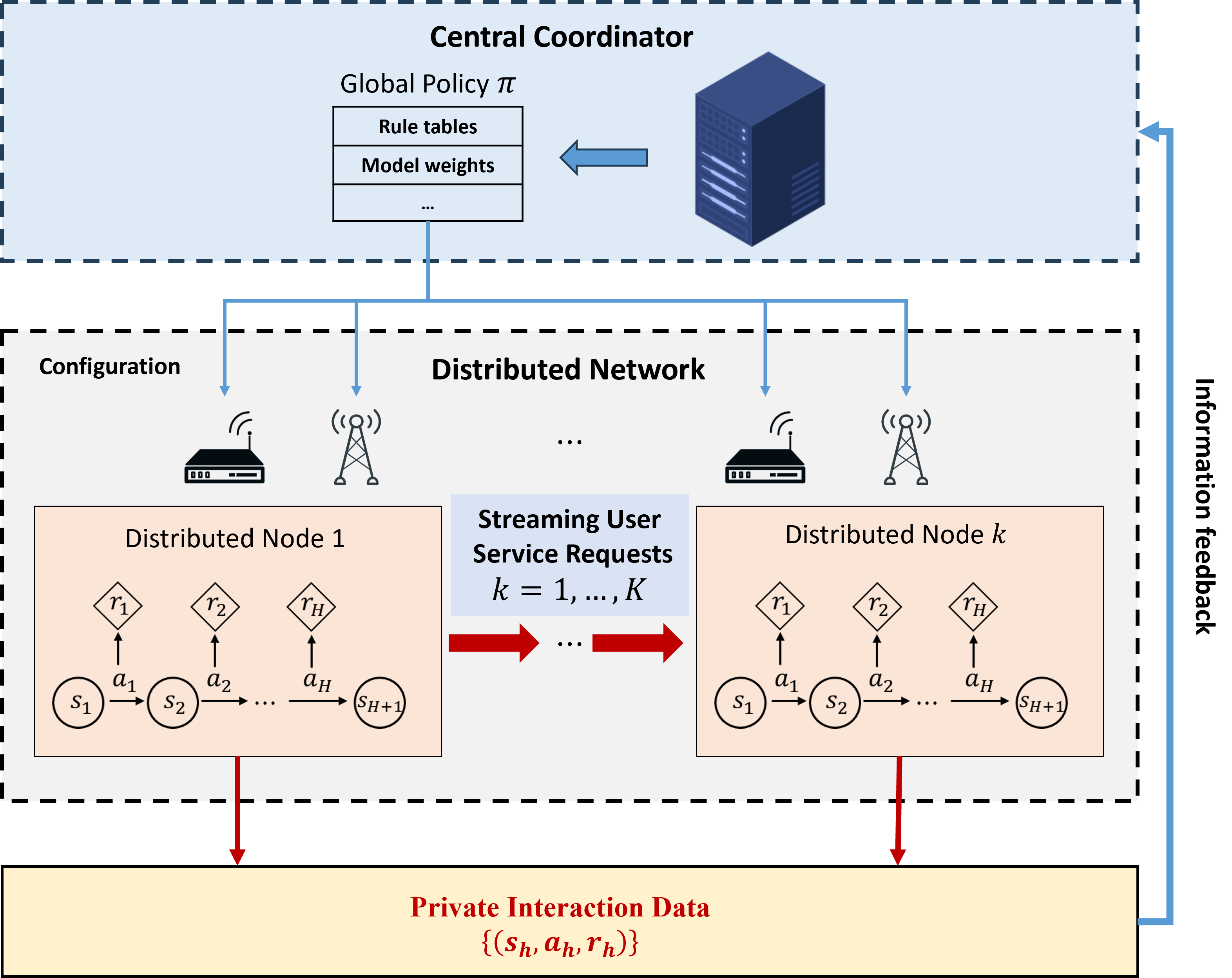}
  \caption{An illustration of online reinforcement learning in a distributed network. A central coordinator performs "centralized learning", and provides service for online streaming user requests in distributed nodes by policy deployment. Each local node generates episodic interaction data. This sensitive interaction data is then (privately) transmitted back to the coordinator for global policy updates.}
  \label{fig: networkRL case}
\end{figure}

However, a crucial concern for its deployment is that the interaction data from each node usually contains sensitive information (e.g., local traffic patterns and spectrum usage habits in the DSA instance).
Without privacy protection mechanisms in place, the learning agent can memorize sensitive information from the interaction history~\cite{zhang2016privacy}, leaving the system to various privacy attack~\cite{mo2024security}. 
Existing Differential Privacy (DP)~\cite{dwork2014algorithmic} models for such sequential decision-making \cite{vietri2020private,baidifferentially,garcelon2021local,chowdhury2022differentially} are often inadequate: the centralized privacy model (JDP)~\cite{vietri2020private,baidifferentially} requires a fully trusted server to collect raw trajectories, creating a single point of failure. 
Conversely, the local privacy model (LDP) \cite{garcelon2021local,chowdhury2022differentially} avoids this risk by privatizing data on each node's device, but at the cost of severe performance degradation (e.g., a multiplicative regret term $\mathcal{O}(\sqrt{K}/\epsilon)$), which is often unacceptable for performance-critical network optimization. 
This naturally leads to the following question: 
\textit{Can a finer trade-off between privacy and regret in RL be achieved, i.e., achieving utility comparable to that of centralized privacy models but without relying on a trusted central server?}

\begin{figure}[!htbp]
  \centering
  \includegraphics[width=0.48\textwidth]{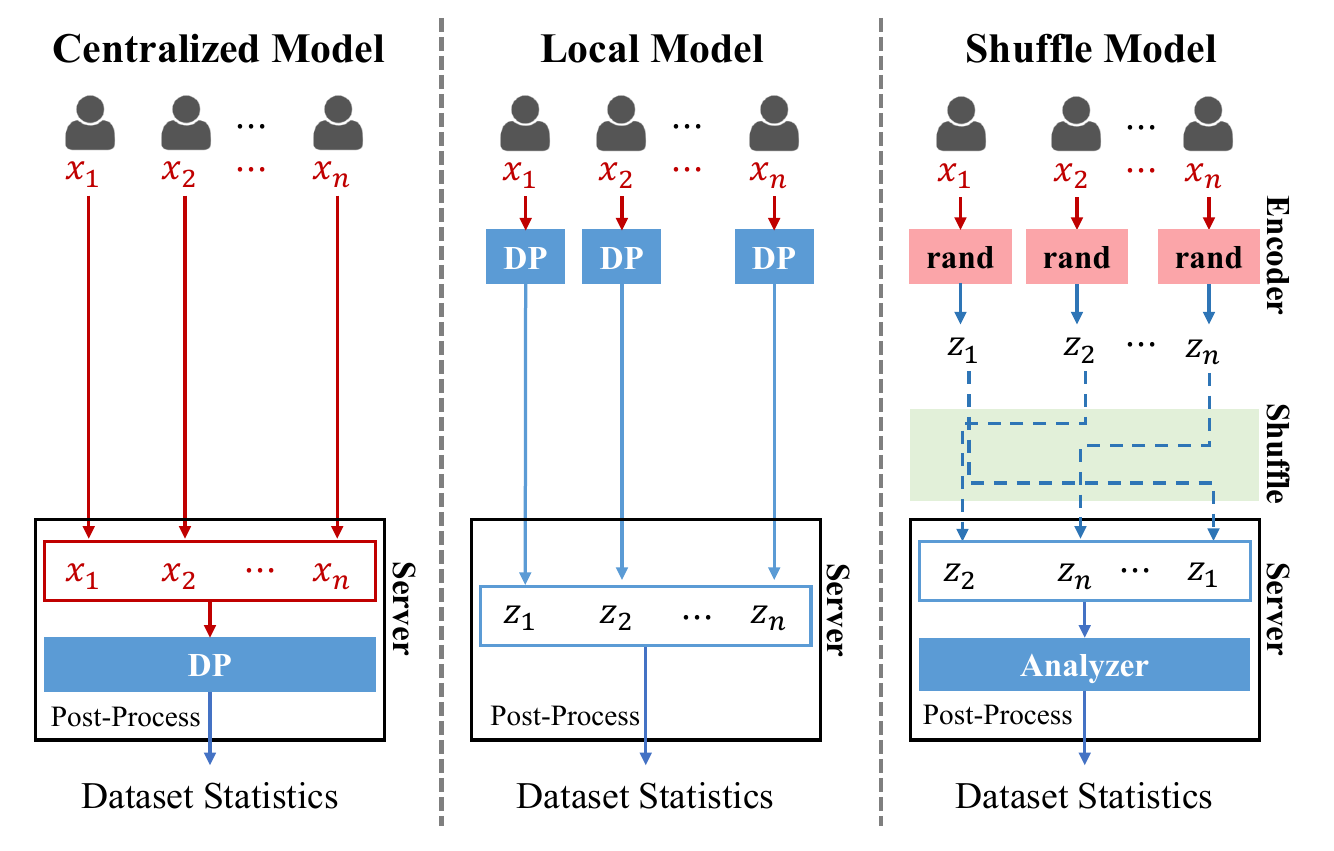}
  \caption{Illustration of the operational workflows in the centralized, local, and shuffle privacy models.}
  \label{fig: privacy-model}
\end{figure}

Motivated by these, this paper leverages an intermediate trust model of differential privacy, known as \emph{shuffle differential privacy} (SDP)~\cite{cheu2019distributed},
as a resilient, distributed architecture.
As illustrated in Figure~\ref{fig: privacy-model}, this paradigm requires a secure shuffler to anonymize a \emph{batch} of locally randomized data before it reaches the central server, and makes sure that the central server can not distinguish any two users' data.
The shuffler, implemented by trusted cryptographic mixnets~\cite{bittau2017prochlo}, creates a ``privacy blanket''~\cite{balle2019privacy} so that each user can inject much less noise and hide their information in the crowd.
This allows each node to inject significantly less noise than in LDP, enabling the system to achieve high utility comparable to JDP but without the centralized trust assumption, making it highly practical.

However, the batch-processing requirement of SDP forces us to abandon the classical ``per-interaction'' online model and shifts to a ``batched'' online model. 
This privacy-driven ``batching'' constraint poses new algorithmic challenges. 
First, by operating on anonymized batches, it imposes an inherent delay, as the agent receives no feedback until an entire batch of interactions is complete. 
Second, to ensure privacy, the aggregated data is intentionally perturbed with statistical noise. 
This combination of delayed feedback and noisy observations severely complicates the crucial exploration-exploitation trade-off.
At the same time, this constraint also creates an opportunity to solve a second critical engineering bottleneck in this streaming service model: the policy switching cost. 
Standard online algorithms require per-interaction ``reconfiguration'' updates (i.e., after every $H$-step session), which is operationally infeasible in large-scale networks due to the heavy deployment and communication cost~\cite{8500100}. 
The ``batched'' model, necessitated by SDP, naturally allows for batched policy updates, providing a path to solve this engineering problem.

In this paper, we overcome these challenges by introducing \emph{Shuffle Differentially Private Policy Elimination} (SDP-PE), the first online RL framework co-designed to exploit this synergy. 
It uses an exponential schedule that both satisfies the batching needs of SDP and solves the engineering cost of policy switching. 
We provide both rigorous theoretical analysis and empirical validations for our method, and the main contributions are summarized as follows:
\begin{itemize}
\item Problem Formulation: We formalize the ``distributed-to-global'' network optimization problem for streaming online services (like DSA) under the dual considerations of SDP and low policy switching cost.
\item Algorithm Design: 
Based on the \emph{policy elimination} idea~\cite{qiao2022sample}, we develop two key strategies tailored for the shuffle model:
a) a staged learning process with exponentially increasing batch sizes. 
This structure allows for the iterative refinement of an active policy set, and is critical for managing the exploration-exploitation trade-off under the delayed feedback inherent to the shuffle model.
b) a “forgetting” mechanism that uses only data from the current stage for model updates. 
This approach prevents the accumulation of privacy-induced noise across the learning stages. 

\item Novel SDP Privatizer for RL: 
We design a new SDP Privatizer that provides strong utility guarantees for the statistical estimates required in RL. 
Built upon a shuffle-private binary summation protocol, our privatizer is designed for seamless integration into the SDP-PE framework, ensuring both rigorous privacy and high performance.

\item Theoretical Guarantees: We prove that SDP-PE obtains a regret of $\widetilde{\cO}\big(\sqrt{\statesize\actionsize\horizontotal^3\episodetotal} + \frac{\statesize^3\actionsize\horizontotal^6}{\pripara}\big)$,\footnote{Here $\tilde{\mathcal{O}}(\cdot)$ hides terms that are poly-logarithmic in $\episodetotal$.}
where $\statesize,\actionsize,\episodetotal,\horizontotal,\pripara$ represent the number of states, actions, episodes, the episode length, and the privacy budget, respectively.
This result matches the information-theoretic lower bound in standard episodic RL up to lower-order terms, demonstrating a superior privacy-performance trade-off.

\item Numerical Validations: We empirically demonstrate our framework's (a) superior privacy-utility trade-off, (b) massive (orders of magnitude) reduction in network switching/deployment costs, which also further corroborates our theoretical findings.
\end{itemize}


We summarize the comparison of best-known results under three distinct privacy guarantees in Table~\ref{table: contributions}: $\pripara$-JDP, $\pripara$-LDP, and $(\pripara,\priconf)$-SDP,
where $\pripara$ and $\priconf$ refer to the privacy budget and probability of privacy failure, respectively.
Due to space limitations, some proof details are deferred to the appendix.

\begin{table}[!htbp]
\begin{center}
\caption{Results Comparison}
\label{table: contributions}
\renewcommand{\arraystretch}{2}
\small
\centering
\begin{tabular}{|c|c|c|}
\hline
Algorithm & Privacy model & Best-known regret bounds \\ \hline
\multirow{2}{*}{\makecell{DP-UCBVI \\~\cite{qiao2023near}}} & $\pripara$-JDP & $\widetilde{\cO}\rbr{\sqrt{\statesize\actionsize\horizontotal^3\episodetotal} + \frac{\statesize^2\actionsize\horizontotal^3}{\pripara}}$ \\ \cline{2-3}
& $\pripara$-LDP 
& $\widetilde{\cO}\rbr{\sqrt{\statesize\actionsize\horizontotal^3\episodetotal} + \frac{\statesize^2\actionsize\horizontotal^3\sqrt{\episodetotal}}{\pripara}}$ \\ \hline
\makecell{SDP-PE \\ (Ours) }
& $(\pripara,\priconf)$-SDP & $\widetilde{\cO}\rbr{\sqrt{\statesize\actionsize\horizontotal^3\episodetotal} + \frac{\statesize^3\actionsize\horizontotal^6}{\pripara}}$ \\ \hline
\end{tabular}
\end{center}
\end{table}

The remainder of this paper is organized as follows. 
In Section~\ref{sec: Preliminary}, we provide the preliminaries of episodic RL and adapt Shuffle DP for RL and formulate the problem. 
Section~\ref{sec: Algorithm} presents the designed policy elimination algorithm under SDP constraints.
The regret and privacy guarantees are presented in Sections~\ref{sec: Regret Guarantee} and \ref{sec: Privacy Guarantees}. 
The numerical experiment is shown in Section~\ref{sec: experiment}.
Discussion and related work are presented in Section~\ref{sec: discussion} and Section~\ref{sec: related work}, respectively.
Finally, Section~\ref{sec: Conclusion and Future Work} presents the conclusion.

\section{Preliminaries and Problem formulation}
\label{sec: Preliminary}
In this section, we review the necessary background on episodic reinforcement learning and shuffle differential privacy, and then formally state the private RL problem.
\subsection{Episodic Reinforcement Learning}
\label{subsec: episodic rl}
An episodic Markov decision process (MDP) is defined by a tuple $\left(\statespace, \actionspace, \horizontotal, \{\transeasy_\horizon\}_{\horizon=1}^\horizontotal, \{\rewdis_\horizon\}_{\horizon=1}^\horizontotal, d_1 \right)$, 
where $\statespace,\actionspace$ are state and action spaces with respective cardinalities $\statesize$ and $\actionsize$, and $\horizontotal$ is the episode length. 
At step $\horizon$, the transition function $\transeasy_\horizon(\cdot\vert\state,\action)$ takes a state-action pair and returns a distribution over states, 
the reward distribution $\rewdis_\horizon(\state,\action)$ is a distribution over $\{0,1\}$ with expectation $\reward_\horizon(\state,\action)$, and $d_1$ is the distribution of initial state.\footnote{For simplicity, we assume that the rewards are binary. However, our results naturally extend to $[0,1]$, by replacing our private binary summation mechanism (defined in Section \ref{sec: Privacy Guarantees}) with a private summation mechanism for real numbers in $[0, 1]$ with similar guarantees.}
A deterministic policy is defined as a collection $\policy=(\policy_1,\dots, \policy_\horizontotal)$ of policies $\policy_\horizon: \statespace\rightarrow\actionspace$.
The value function $\valuef_\horizon^\policy$ and Q function $\qvaluef_\horizon^\policy$ are defined as: $\valuef_\horizon^\policy(\state) = \EE_\policy[\sum_{t=\horizon}^\horizontotal \reward_t\vert \state_\horizon=\state], \qvaluef_\horizon^\policy(\state,\action) = \EE_\policy[\sum_{t=\horizon}^\horizontotal \reward_t\vert \state_\horizon,\action_\horizon=\state,\action], \forall\state,\action\in\statespace\times\actionspace$.
There exists an optimal policy $\policy^\star$ such that $\valuef_\horizon^\star(\state) =  \valuef_\horizon^{\policy^\star}(\state) = \max_\policy \valuef_\horizon^\policy(\state)$ for all $\state,\horizon\in\statespace\times[\horizontotal]$.\footnote{For a positive integer $n$, we define $[n]\!:=\!\{1,\ldots,n\}$.} 
The Bellman (optimality) equation follows $\forall \horizon\in[\horizontotal]: \qvaluef_\horizon^\star(\state,\action) =\reward_\horizon(\state,\action) + \max_{\action'}\expect_{\state'\sim\transeasy_\horizon(\state,\action)}[\valuef_{\horizon+1}^{\star}(\state')].$
The optimal policy is the greedy policy: $\policy^\star_\horizon(\state
) = \argmax_{\action} \qvaluef_\horizon^\star(\state,\action),\forall \state\in\statespace.$
For generalization, we define the value function of $\policy$ under MDP transition $p$ and reward function $r'$ as $\valuef^\pi(r',p)$.

We consider a central agent tasked with interacting with a sequence of $\episodetotal$ online user requests in the network system.
Each episode $\episode\in[\episodetotal]$ is modeled as an interaction with user request $\user_\episode\in\userspace$, where $\userspace$ is the user space. 
For each episode $\episode\in[\episodetotal]$, the learning agent determines its current policy $\policy_\episode$ and deploys it to user $\user_\episode$ for execution. 
The output of the execution, a trajectory $\traj_\episode= \rbr{\state_\horizon^\episode,\action_\horizon^\episode,\reward_\horizon^\episode}_{\horizon\in[\horizontotal]}$, represents the sensitive data to be protected.
The trajectory is then sent back to the learning agent for policy updating.
We measure the performance of a learning algorithm by its cumulative regret after $\episodetotal$ episodes, 
\begin{equation}
    \regret := \sum_{\episode=1}^\episodetotal \sbr{ \valuef_1^\star(\state_1^\episode) - \valuef_1^{\policy_\episode}(\state_1^\episode)}, \quad \state_1^\episode\sim d_1.
\end{equation}
A secondary metric is the policy switching cost~\cite{bai2019provably}, which measures how many times the agent changes the policy,
\begin{equation}
    N_{switch}:= \sum_{\episode=1}^{\episodetotal-1} \mathbf{1}\{\policy_\episode\neq \policy_{\episode+1}\}.
\end{equation}



\subsection{Differential Privacy and Shuffle Model}
\label{subsec: dp and shuffle model}
Differential privacy provides a formal guarantee of data privacy.
Consider a general case where $\trajset$ is the data universe, and we have $n$ \emph{unique} users.  
We say $\batchdata, \batchdata' \in \trajset^n$ are neighboring batched datasets if they only differ in one user's data for some $i\in[n]$. 
Then, we have the standard definition of DP~\cite{dwork2014algorithmic}:
\begin{definition}[Differential Privacy (DP)] \label{def: DP}
    For $\pripara,\priconf > 0$, a randomized mechanism $\cM$ is $(\pripara,\priconf)$-DP if for all neighboring datasets $\batchdata, \batchdata'$ and any event $E$ in the range of $\cM$, we have
    $$\prob[\cM(\batchdata)\in E]  
    \leq \exp(\pripara) \cdot \prob[\cM(\batchdata')\in E] + \priconf.$$
\end{definition}
Here, the randomized mechanism $\cM$ is an algorithm that takes a dataset as input and uses internal randomness to produce a privatized output, such as a statistic.
The special case of $(\pripara,0)$-DP is also called \emph{pure DP}, whereas, for $\priconf>0$, $(\pripara,\priconf)$-DP is referred to as \emph{approximate DP}.

The shuffle model is an enhanced DP protocol $\protocol = (\randomizer,\shuffler,\analyzer)$ involving three components, as shown in Figure~\ref{fig: privacy-model}: 
(i) a (local) random encoder $\randomizer$ at each user's side; 
(ii) a secure shuffler $\shuffler$; and 
(iii) an analyzer $\analyzer$ at the central server (i.e., the learning agent).
In this model, each user $\user_i$ first locally randomizes their data $\batchdata_i$ via the encoder $\randomizer$, and sends the resulting messages $\randomizer(\batchdata_i)$ to the shuffler $\shuffler$.
The shuffler $\shuffler$ permutes messages from this batch users at random and then reports $ \shuffler(\randomizer(\batchdata_1),\dots,\randomizer(\batchdata_n))$ to the analyzer.
The analyzer $\analyzer$ then aggregates the received messages and outputs desired statistics.
In this protocol, the users trust the shuffler but not the analyzer, and the goal is to ensure the output of the shuffler $\shuffler$ on two neighboring datasets is indistinguishable in the analyzer's view.
Here, we define the mechanism $(\shuffler\circ\randomizer^n)(\batchdata) = \shuffler(\randomizer(\batchdata_1),\dots,\randomizer(\batchdata_n))$, where $\batchdata\in\batchdataspace^n$.
Thus, we have the definition of standard shuffle DP~\cite{cheu2019distributed}:
\begin{definition}[Shuffle Differential Privacy (SDP)] 
\label{def: SDP}
    A protocol $\protocol = (\randomizer,\shuffler,\analyzer)$ for $n$ users is $(\pripara,\priconf)$-SDP if the mechanism $\shuffler\circ\randomizer^n$ satisfies $(\pripara,\priconf)$-DP.
\end{definition}

In a sequential data-generation setting, this protocol will be applied across multiple batches.
To protect the sequence of outputs across all $\batchnumber$ batches, we define the (composite) mechanism $\cM_\protocol = (\shuffler\circ\randomizer^{\batchlength_1},\dots,\shuffler\circ\randomizer^{\batchlength_\batchnumber}),$ 
where each mechanism $\shuffler\circ\randomizer^{\batchlength_\batchindx}$ operates on a dataset of $\batchlength_\batchindx$ users.
This leads to the following definition $\batchnumber$-batch SDP \cite{cheu2019distributed}:

\begin{definition}[$\batchnumber$-batch SDP]
A $\batchnumber$-batch shuffle protocol $\protocol$ is $(\pripara,\priconf)$-SDP if the mechanism $\cM_\protocol$ satisfies $(\pripara,\priconf)$-DP.
\end{definition} 

\subsection{Problem Formulation: Shuffle-Private RL}
\label{subsec: problem formulation: shuffle private rl}
Building upon the above definitions, we formally define the problem addressed in this paper. 
We consider an agent interacting with an unknown episodic MDP as described in Section~\ref{subsec: episodic rl}, but under the communication constraint of the $\batchnumber$-batch shuffle privacy defined in Section~\ref{subsec: dp and shuffle model}.

The core of our problem lies in the information feedback loop.
The learning process is partitioned into $\batchnumber$ batches.
In each batch $\batchindx$, the agent deploys some policy to a group of $\batchlength_\batchindx$ users, and each user's resulting interaction trajectory is considered their private data. 
Crucially, the agent does not observe the raw interaction trajectories.
Instead, they are processed by a shuffle protocol, and the agent only receives a privatized, aggregated output for that batch, which it then uses to select policies for the subsequent batch.
The central challenge is to design an algorithm that achieves a dual objective: 
\begin{itemize}
    \item Learning performance: Minimize the cumulative regret $\regret$.
    \item System privacy: Guarantee that the entire sequence of outputs observed by the agent satisfies a $\batchnumber$-batch $(\pripara, \priconf)$-SDP guarantee.
\end{itemize}

This formulation captures the fundamental tension of our work: the agent must learn to make effective sequential decisions despite receiving only delayed, batched, and noisy information.

\section{Algorithm Design} \label{sec: Algorithm}
In this section, we introduce a generic algorithmic framework, \emph{Shuffle Differentially Private Policy Elimination} (SDP-PE, Algorithm \ref{algo: Private Batch-based Policy Elimination}).
Unlike existing private RL algorithms that rely on the \emph{optimism in the face of uncertainty} principle \cite{vietri2020private,garcelon2021local}, SDP-PE builds on the \emph{policy elimination} idea from \cite{qiao2022sample,zhang2022near}, extending action elimination from bandits~\cite{evendar06a} to RL.

As illustrated in Figure~\ref{fig: SDP-PE}, SDP-PE operates in a sequence of stages with exponentially increasing batch sizes, where $\batchlength_\batchindx:=2^\batchindx$ for $\batchindx=1,\dots,\batchnumber$ and $\batchnumber = \cO(\log\episodetotal)$. 
Within each stage, the algorithm maintains a set of ``active" policies ($\policyset$) and follows a three-step procedure: \textit{Crude Exploration}, \textit{Fine Exploration}, and \textit{Policy Elimination}.
The process begins by broadly exploring the environment to build a preliminary model. 
This model then guides a more refined exploration phase to gather higher-quality data and keep low regret. 
Finally, using this refined data, it constructs private confidence bounds on the value of each active policy and eliminates any policy that is provably sub-optimal. 
This cycle of explore-refine-eliminate repeats across stages. 
As the algorithm proceeds and the batch sizes grow, the confidence intervals shrink, ensuring that only near-optimal policies remain.

The high-level, three-step structure of our algorithm is inspired by the \textsc{APEVE} framework of \cite{qiao2022sample}.
However, our objective of minimizing regret under shuffle privacy, rather than minimizing policy switches, necessitates several fundamental innovations in the design as detailed below.

\emph{1) Robust estimation via a privatized absorbing MDP.}
We adapt the technique of using an absorbing MDP from \textsc{APEVE} to handle rarely visited states. 
The key difference is that in our setting, the set of ``infrequently-visited" tuples must be identified from private counts (cf.~Assumption \ref{assp: private counts}).   
This demands a careful construction that is robust to statistical noise, allowing the algorithm to focus its learning on the core, reachable parts of the environment.

\emph{2) Optimal trade-offs via exponential batch sizing and ``forgetting".} 
This principle is a core departure from prior work, designed to manage the dual trade-offs of exploration vs.~exploitation and privacy vs.~regret.

\underline{Exploration vs.~exploitation:}
    The schedule of increasing batch sizes manages this trade-off. 
    Early, smaller batches prioritize exploration by enabling frequent policy updates to quickly move away from a poor initial policy. 
    As the algorithm's confidence grows, larger batches are used for longer periods of confident exploitation, leveraging higher-quality data to reliably eliminate sub-optimal policies.
    
\underline{Regret control:} Crucially, the exponential growth rate (i.e., $L_b \propto 2^b$) is mathematically essential. 
    Our analysis shows that this specific schedule ensures two things simultaneously: the non-private regret term matches the lower bound in standard episodic RL, while the privacy cost in regret becomes a logarithmic and additive term. 
    Any other schedule would result in a suboptimal privacy cost.
    
\underline{Forgetting:} 
This batching strategy is paired with a ``forgetting" mechanism, using only data from the current stage for estimates, to prevent the accumulation of noise from previous, less accurate stages.

\emph{3) A noise-robust fine exploration policy.}
A third innovation of SDP-PE lies in policy design for the fine exploration phase. 
Unlike its counterpart in \textsc{APEVE}, our policy is specifically designed to guarantee uniform data coverage over all active policies even in the presence of privacy noise. 
This is critical for deriving the tighter confidence bounds needed to achieve our final regret.

The following sections will now elaborate on the components of this framework, detailing the private counting mechanisms and each of the three procedural steps.

\begin{figure}[!htbp]
  \centering
  \includegraphics[width=0.5\textwidth]{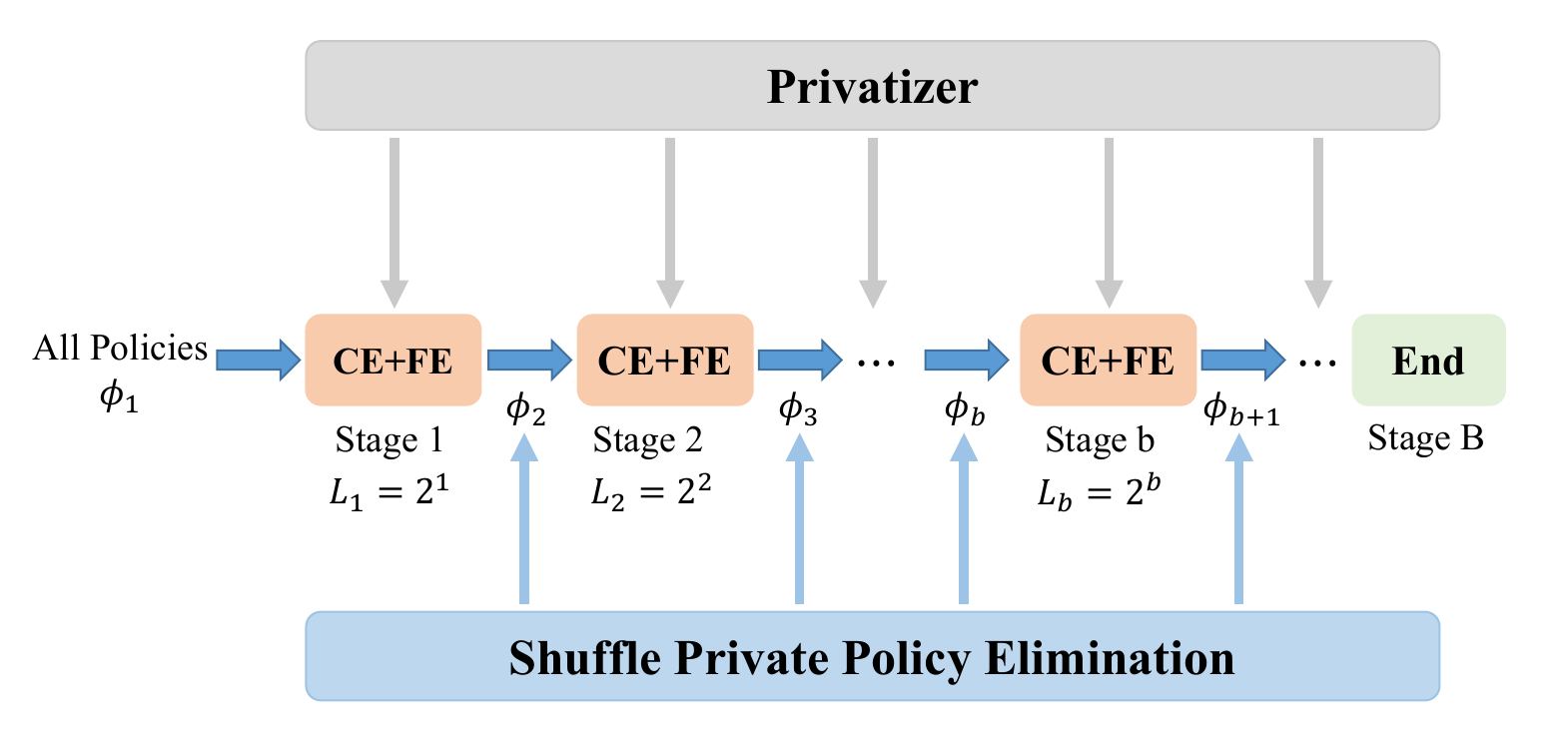}
  \caption{Overview of the SDP-PE algorithm (Algorithm~\ref{algo: Private Batch-based Policy Elimination}). The central coordinator deploys a global policy to distributed nodes, which process batched user requests. Private interaction data is then shuffled and transmitted back for policy value updates and elimination.}
  \label{fig: SDP-PE}
\end{figure}

\begin{algorithm}[!htbp]
{
\caption{Shuffle DP Policy Elimination (SDP-PE)} 
\textbf{Parameters:} Episode number $\episodetotal$, universal constant $C$, failure probability $\delta$, privacy budget $\pripara > 0$ and a Privatizer. \newline
\textbf{Initialization:} \sloppy $\policyset_1 = \{\text{all deterministic policies}\}$, $\iota = \log(2\horizontotal\actionsize\episodetotal/\delta)$. 
\text{Set precision levels $\confcountxa$ for Privatizer.}

\begin{algorithmic}[1]
\FOR{$\batchindx=1$ to $\batchnumber$}
\IF{$3(\sum_{i=1}^\batchindx \batchlength_i)\geq\episodetotal$}
\STATE $\batchlength_\batchindx = \frac{\episodetotal - 3(\sum_{i=1}^\batchindx \batchlength_i)}{3}$. (o.w. $\batchlength_\batchindx = 2^\batchindx$)
\ENDIF
\STATE $\infreqtuples^\batchindx,\!\crutrans,\!\policy_\batchindx\! =\! \text{CrudeExploration}(\policyset_\batchindx,\batchlength_\batchindx,\pripara,\!\text{Privatizer}).$
\STATE $\reftrans,\rewprieasy^\batchindx = \text{FineExploration}(\infreqtuples^\batchindx,\crutrans,\policyset_\batchindx, \policy_\batchindx, \batchlength_\batchindx,\pripara, 
    $ $\text{Privatizer})$
\STATE $\elimpolicyset_\batchindx = \emptyset$
\FOR{$\policy \in \policyset_\batchindx$}
\IF{$\sup_{\mu\in\policyset_\batchindx}\valuef^{\mu}(\rewprieasy^\batchindx,\reftrans) - \valuef^\policy(\rewprieasy^\batchindx,\reftrans) \geq 2C\big(\sqrt{\frac{\statesize\actionsize\horizontotal^3\iota}{\batchlength_\batchindx}}+\frac{\statesize^3\actionsize\horizontotal^5\confcountxa\iota}{\batchlength_\batchindx}\big)$}
\STATE Update $\elimpolicyset_\batchindx\leftarrow \elimpolicyset_\batchindx \cup \{\policy\}$.
\ENDIF
\ENDFOR
\STATE $\policyset_{\batchindx+1} \leftarrow \policyset_{\batchindx} \backslash \elimpolicyset_\batchindx$.
\ENDFOR
\end{algorithmic}
\label{algo: Private Batch-based Policy Elimination}}
\end{algorithm}

\subsection{Model Estimation via Private Counts in Algorithm \ref{algo: Private Batch-based Policy Elimination}}
\label{sec: counts in algorithm SDP-PE}
SDP-PE employs a \emph{model-based} approach, relying on a ``forgetting'' principle to estimate the MDP model in each stage.
This means that only the data from the current batch is used for model estimation.
While the specific dataset construction varies between the crude exploration and fine exploration steps, as will be detailed later,
the general process for estimating the model from a batch trajectories is as follows.

Given a batch dataset $\batchdata$ from $n$ users,
we define the true visitation counts and cumulative rewards at each step $\horizon$,
$\visitxaxtotaleasy_\horizon(\state,\action,\state')\!:=\!\sum_{i=1}^{n} \ind\{\state_\horizon^{i},\action_\horizon^{i},\state_{\horizon+1}^{i} = \state,\action,\state'\}$, 
similarly $\visitxatotaleasy_\horizon(\state,\action)\!\!:=\!\!\sum_{\state^\prime\in\statespace} \visitxaxtotaleasy_\horizon(\state,\action,\state')$, 
and $\rewcumeasy_\horizon(\state,\action)\!:=\!\sum_{i=1}^{n} \ind\{\state_\horizon^{i},\action_\horizon^{i} = \state,\action\} \cdot \reward_\horizon^{i}$. 
These empirical counts are then processed by a shuffle Privatizer, which then releases their privatized versions, denoted as $\visitxatotalprieasy_\horizon(\state,\action,\state'),\visitxaxtotalprieasy_\horizon(\state,\action)$, and $\rewcumprieasy_\horizon(\state,\action)$.
As formalized in Assumption~\ref{assp: private counts}, these private counts are guaranteed to closely approximate the true value, which is essential for reliable learning.
The detailed implementation of such Privatizer is justified in Section~\ref{sec: Privacy Guarantees}.

\begin{assumption}[Private counts] \label{assp: private counts}
For any privacy budget $\pripara >0$ and failure probability $\delta\in(0,1)$, the private counts returned by Privatizer satisfy,
for some $\confcountxax > 0$, with probability at least $1-3\delta$, over all $\rbr{\horizon, \state,\action,\state^\prime}$,
$\vert\visitxatotalprieasy_\horizon(\state,\action)\!-\!\visitxatotaleasy_\horizon(\state,\action)\vert$ $\leq\!\! \confcountxa, \vert\visitxaxtotalprieasy_\horizon(\state,\action,\state')\!-\!\visitxaxtotaleasy_\horizon(\state,\action,\state')\vert\!\!\leq\!\!\confcountxax$, 
$\vert\rewcumprieasy_\horizon(\state,\action)-\rewcumeasy_\horizon(\state,\action)\vert \leq \confrewf$
and $\visitxatotalprieasy_\horizon(\state,\action)$ $ =\sum_{\state^\prime\in\statespace}$ $\visitxaxtotalprieasy_\horizon(\state,\action,\state')$ $\geq$ $\visitxatotaleasy_\horizon(\state,\action)$, $\visitxaxtotalprieasy_\horizon(\state,\action,\state') > 0$.
\end{assumption}

Based on these private counts, we define the estimated transition and reward functions for current stage:
\begin{equation}    
\label{eq: private transition estimate}
\transprieasy_\horizon\rbr{\state^\prime\vert\state,\action}:= \frac{\visitxaxtotalprieasy_\horizon(\state,\action,\state')}{\visitxatotalprieasy_\horizon(\state,\action)}, \rewprieasy_\horizon(\state,\action):= \frac{\rewcumprieasy_\horizon(\state,\action)}{\visitxatotalprieasy_\horizon(\state,\action)}.
\end{equation}
A key property of our Privatizer, enforced by Assumption \ref{assp: private counts}, is that it maintains summation consistency, i.e., $\visitxatotalprieasy_\horizon(\state,\action) = \sum_{\state^\prime\in\statespace} \visitxaxtotalprieasy_\horizon(\state,\action,\state')$. 
This construction guarantees that our estimated transition function, $\transprieasy_\horizon\rbr{\cdot\vert\state,\action}$, is always a valid probability distribution.

\subsection{Crude Exploration in Algorithm \ref{algo: Crude Exploration}}
\begin{algorithm}[!htbp]
\caption{Crude Exploration} 
\textbf{Input:} Policy set $\policyset$, number of episodes $\batchlength$, privacy budget $\pripara$ and a Privatizer. \\
\textbf{Initialization:} 
$\batchlength_0 = \frac{\batchlength}{\horizontotal}, 
C_1=6, 
\infreqtuples=\emptyset,\iota = \log(2\horizontotal\actionsize\episodetotal/\delta)$.
$1_{\horizon,\state,\action}$ is a reward function $\reward'$ where $\reward'_{\horizon'}(\state',\action') = \ind\{(\horizon',\state',\action')=(\horizon,\state,\action)\}$. $\absorbstate$ is an additional absorbing state. $\tilde{P}^{\text{cru}}$ is a transition function over extended space $\statespace\cup\{\absorbstate\}\times\actionspace$, initialized arbitrarily.
\\
\textbf{Output:} Infrequent tuples $\infreqtuples$, and crude estimated transition function $\tilde{P}^{\text{cru}}$, mixture policy $\policy_0$
\begin{algorithmic}[1]
\FOR{$\horizon=1$ to $\horizontotal$}
\STATE Set data set $\batchdata^\horizon = \emptyset$.
\FOR{$(\state,\action)\in\statespace\times\actionspace$}
\STATE $\policy^{\text{cru}}_{\horizon,\state,\action} = \argmax_{\policy\in\policyset} \valuef^\policy(1_{\horizon,\state,\action},\tilde{P}^{\text{cru}})$.
\ENDFOR
\STATE Run $\policy^{\text{cru}}_{\horizon}=$ uniform mixture of $\{\policy^{\text{cru}}_{\horizon,\state,\action}\}_{(\state,\action)}$ for $\batchlength_0$ episodes, and add the trajectories into dataset $\batchdata^{\horizon}$.
\STATE Send batch dataset $\batchdata^{\horizon}$ to the Privatizer.
\STATE Receive counts $\visitxaxtotalprieasy^{\text{cru}}_\horizon(\state,\action,\state')$ and $\visitxaxtotalprieasy^{\text{cru}}_\horizon(\state,\action)$ for all $(\state,\action,\state')$ in $\horizon$-th horizon from Privatizer.
\STATE $\infreqtuples\! = \infreqtuples\cup\{(\horizon,\state,\action,\state')\vert \visitxaxtotalprieasy^{\text{cru}}_\horizon(\state,\action,\state')\! \leq\! C_1\confcountxa\horizontotal^2\iota\}\!.$
\STATE $\tilde{P}^{\text{cru}} = \text{Estimate Transition}(\visitxatotalprieasy^{\text{cru}}_\horizon,\infreqtuples,\absorbstate,\horizon,\tilde{P}^{\text{cru}}).$
\ENDFOR
\STATE Policy $\policy_0 \leftarrow$ uniform mixture of $\{\policy^{\text{cru}}_{\horizon,\state,\action}\}_{(\horizon,\state,\action)}$.
\RETURN $\{\infreqtuples,\tilde{P}^{\text{cru}},\policy_0\}$.
\end{algorithmic}
\label{algo: Crude Exploration}
\end{algorithm}

The first step, crude exploration (Algorithm \ref{algo: Crude Exploration}), is designed to build a preliminary, or ``crude," model of the environment $\crutrans_\horizon(\state'\vert\state,\action)$ for any tuple $(\horizon,\state,\action,\state')$. 
The main challenge is that some state-action pairs may be too rarely visited to gather sufficient statistics, especially under privacy constraints. 
Our strategy is to formally partition the system's state space by identifying a set of infrequently-visited tuples, denoted as $\infreqtuples$.
The key idea behind is that for the tuples not in $\infreqtuples$, we can get accurate estimates, and for tuples in $\infreqtuples$, they have little influence on the value estimate.

To identify this set, we perform a layer-wise exploration.
For each step $\horizon$ of episode, we explore every state-action pair $(\state,\action)$ by constructing a policy, $\policy^{\text{cru},\batchindx}_{\horizon,\state,\action}$, that maximizes the estimated probability of visiting that pair under the current model.
We then form a single exploration policy for the layer $\policy^{\text{cru},\batchindx}_{\horizon}$, a uniform mixture of these specialized policies.
This mixture policy is executed and the resulting trajectories $\batchdata^\horizon$ are batched and sent to the Privatizer.
These returned private counts are then used to empirically identify all tuples visited fewer than $\cO(\confcountxa\horizontotal^2\iota)$ times, which form the set $\infreqtuples$. 

This set is then used to construct an absorbing MDP, a technique adapted from \cite{qiao2022sample}.
As fomally defined below, this method isolates the hard-to-reach states by redirecting their transition probabilities to a single absorbing state, $\absorbstate$.
\begin{definition}[Abosorbing MDP $\abtrans$]
\label{def: absorbing mdp}
    Given $\infreqtuples$ and $\transeasy$, $\forall(\horizon,\state,\action,\state')\notin\infreqtuples$, let $\abtrans_\horizon(\state'\vert\state,\action)=\transeasy_\horizon(\state'\vert\state,\action)$,
    $\forall(\horizon,\state,\action,\state')\in\infreqtuples$, $\abtrans_\horizon(\state'\vert\state,\action)=0$. 
    For any $(\horizon,\state,\action)\in[\horizontotal]\times\statespace\times\actionspace$,we define $\abtrans_\horizon(\absorbstate\vert\absorbstate,\action) = 1$ and $\abtrans_\horizon(\absorbstate\vert\state,\action) = 1 - \sum_{\state'\in\statespace:(\horizon,\state,\action,\state')\notin\infreqtuples} \abtrans_\horizon(\state'\vert\state,\action).$
\end{definition}

The final output is $\crutrans$ derived from Equation~\eqref{eq: private transition estimate}, a private and reliable crude estimate of the absorbing MDP.
This crude model allows the algorithm to focus its learning capacity on the reachable core of the environment and serves as the foundation for the next step.
An auxiliary mixture policy, $\policy_0$ is also generated for use in the subsequent phase.

\subsection{Fine Exploration in Algorithm \ref{algo: Fine Exploration} }
\begin{algorithm}[!htbp]
\caption{Fine Exploration} 
\textbf{Input:} Infrequent tuples $\infreqtuples$, crude estimated transition $\tilde{P}^{\text{cru}}$, policy set $\policyset$, auxiliary policy $\policy_0$, number of episodes $\batchlength$, privacy budget $\pripara$ and a Privatizer. \\
\textbf{Initialization:} 
$\batchdata=\emptyset$.
$1_{\horizon,\state,\action}$ is a reward function $\reward'$ where $\reward'_{\horizon'}(\state',\action') = \ind\{(\horizon',\state',\action')=(\horizon,\state,\action)\}$. 
Initialize refined transition estimate $\tilde{P}^{\text{ref}} = \tilde{P}^{\text{cru}}$. \\
\textbf{Output:} Refined estimated transition function $\tilde{P}^{\text{ref}}$ and reward function $\rewprieasy$.
\begin{algorithmic}[1]
\STATE Construct explorative policy $\policy^{\text{ref}}$ according to Eq.~\ref{eq: policy for fine exploration}.
\STATE Run $\policy^{\text{ref}}$ for $\batchlength$ episodes and $\policy_0$ for $\batchlength$ episodes, add trajectories to dataset $\batchdata$.
\STATE Send batch dataset $\batchdata$ to the Privatizer.
\STATE Receive private counts $\visitxaxtotalprieasy^{\text{ref}}_\horizon(\state,\action,\state')
$, $\visitxaxtotalprieasy^{\text{ref}}_\horizon(\state,\action)$ and $\rewcumprieasy^{\text{ref}}_\horizon(\state,\action,\state')$ for all $(\horizon, \state,\action,\state')$ from Privatizer.
\FOR{$\horizon\in[\horizontotal]$}
\STATE $\tilde{P}^{\text{ref}} = \text{Estimate Transition}(\visitxatotalprieasy^{\text{ref}}_\horizon,\infreqtuples,\absorbstate,\horizon,\tilde{P}^{\text{ref}}).$
\STATE $\rewprieasy = \text{Estimate Rewards}(\rewcumprieasy^{\text{ref}}_\horizon, \visitxatotalprieasy^{\text{ref}}_\horizon, \horizon, \rewprieasy)$.
\ENDFOR
\RETURN $\tilde{P}^{\text{ref}}$, $\rewprieasy$.
\end{algorithmic}
\label{algo: Fine Exploration}
\end{algorithm}

The second step, fine exploration (Algorithm \ref{algo: Fine Exploration}), leverages the crude model $\crutrans$ to design a more efficient and focused exploration strategy. 
The key insight is that $\crutrans$ is reliable enough to construct a uniformly exploration policy, $\policy^{\text{ref},\batchindx}$, as defined in Equation~\eqref{eq: policy for fine exploration}:
\begin{equation}\label{eq: policy for fine exploration}
    \policy^{\text{ref},\batchindx} = \argmin_{\policy\in\policyset_\batchindx} \sup_{\mu\in\policyset_\batchindx} \sum_{\horizon=1}^\horizontotal\sum_{\state,\action}\frac{\valuef^{\mu}(1_{\horizon,\state,\action},\crutrans)}{\valuef^\policy(1_{\horizon,\state,\action},\crutrans)}.
\end{equation}
This policy is specifically designed to minimize the ``worst-case coverage number" over the current set of active policies, which could provide uniform coverage of all remaining policies. 
To generate data, the users execute both this new policy $\policy^{\text{ref},\batchindx}$ and the auxiliary policy $\policy_0$ from crude exploration phase.
The combined dataset $\batchdata$ is then processed by the Privatizer to yield a refined private estimate $\reftrans$ and reward estimate $\rewprieasy^\batchindx$, which are more accurate than their crude counterparts.
Notably, different from the APEVE algorithm~\cite{qiao2022sample}, our fine exploration policy is specifically designed to achieve better policy coverage.
This design choice guarantees a more uniform value estimation accuracy for all candidate policies, leading to tighter confidence bounds and an improved final regret bound.

The final step of each stage is policy elimination, as illustrated in Figure~\ref{fig: policy elimination}. 
We evaluate all active policies in $\policyset_\batchindx$ using the reliable estimates $\reftrans$ and $\rewprieasy^\batchindx$, and update the active policy set by eliminating all policies whose upper confidence bound (UCB) falls below the lower confidence bound (LCB) of any other active policy.
The confidence intervals shrink along with the policy elimination step, ensuring that the optimal policy stays in the active policy set with high probability, which guarantees that the algorithm eventually converges to the optimal policy.
\begin{figure}[!htbp]
  \centering
  \includegraphics[width=0.3\textwidth]{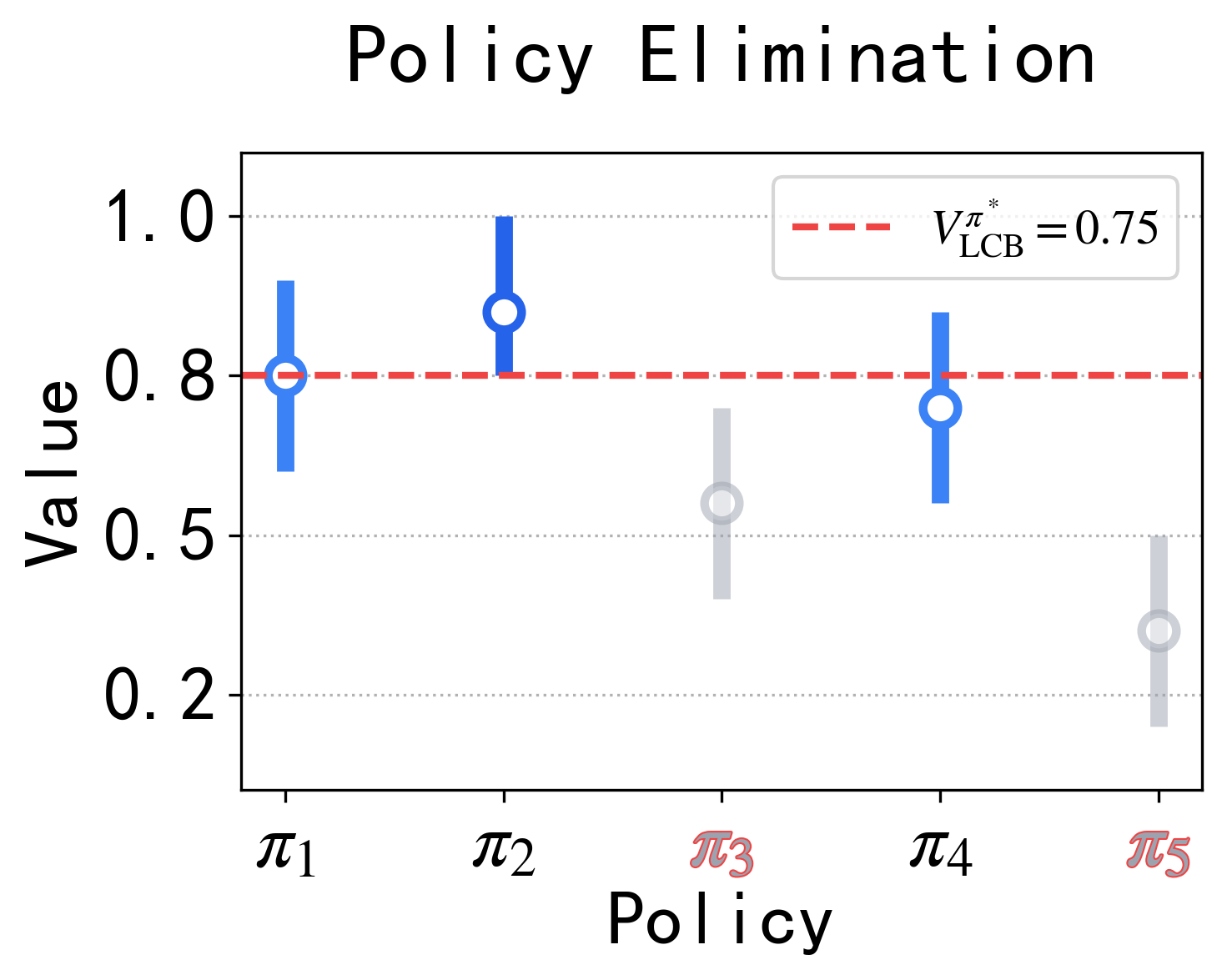}
      \caption{Example of policy elimination. The LCB of the estimated optimal policy $\policy_2$ is $0.75$. Policies $\policy_3$ and $\policy_5$, whose UCB are below this threshold, are identified as sub-optimal and consequently eliminated.}
  \label{fig: policy elimination}
\end{figure}



\section{Performance Guarantee}
\label{sec: Regret Guarantee}
The main theoretical results of this paper are the regret bound and the policy switching cost for our SDP-PE algorithm.
\begin{theorem}[Performance guarantee of SDP-PE]
\label{thm: Regret bound of PBPE}
    For any privacy budget $\pripara>0$ and failure probability $\delta\in(0,1)$, and any Privatizer that satisfies Assumptions \ref{assp: private counts}, with probability at least $1-9\delta$, the regret of SDP-PE (Algorithm \ref{algo: Private Batch-based Policy Elimination}) is
    \begin{equation}\notag
    \begin{aligned}
        \regret \leq \tilde{\cO}\rbr{\sqrt{\statesize\actionsize\horizontotal^3\episodetotal} + \statesize^3\actionsize\horizontotal^5\confcountxa}.
    \end{aligned}
    \end{equation}
    Furthermore, the total number of policy switches during the $\episodetotal$ episodes is $N_{switch} = \cO(\horizontotal\log \episodetotal)$.
\end{theorem}

The bound on the policy switching cost follows directly from the exponential batching schedule and the properties of geometric series.
The core of the regret analysis is to construct a uniform policy evaluation bound that holds for all active policies at each stage of the algorithm with proof details in Appendix~\ref{appendix: proof of regret-theorem}.
In each stage $\batchindx$, the active policy set is $\policyset_\batchindx$.
Assume that for any $\policy\in\policyset_\batchindx$, the value $\valuef^\policy(\reward,\transeasy)$ can be estimated up to an error $\accuracy_\batchindx$ by $\valuef^\policy(\rewprieasy^\batchindx,\reftrans)$, then if we eliminate all policies that are at least $2\accuracy_\batchindx$ sub-optimal in the sense of  $\max_{\policy\in\policyset_\batchindx}\valuef^\policy(\rewprieasy^\batchindx,\reftrans)$, the optimal policy will never be eliminated and all remaining policies will be at most $4\accuracy_\batchindx$ sub-optimal.
Summing the regret across all stages yields the total regret:
\begin{equation}
\begin{aligned}
\label{eq: regret decomposition with stages}
    \regret\leq 3\horizontotal\batchlength_1 + \sum_{\batchindx=2}^\batchnumber 3\batchlength_\batchindx \cdot 4 \accuracy_{\batchindx-1}.
\end{aligned}
\end{equation}
The following lemma gives an upper bound of $\accuracy_{\batchindx}$ using our private estimate $\reftrans$ and $\rewprieasy^\batchindx$ of the absorbing MDP.
\begin{lemma}\label{lemma: total accuracy error}
With probability $1-9\delta$, it holds that for any stage $\batchindx$ and $\policy\in\policyset_{\batchindx}$,
$$\vert\valuef^\policy\!(\reward,\!\transeasy) - \valuef^\policy\!(\rewprieasy^\batchindx\!,\reftrans)\vert \!\leq\!\tilde{\cO}\bigg(\!\sqrt{\frac{\statesize\actionsize\horizontotal^3}{\batchlength_{\batchindx}}} + \frac{\statesize^3\actionsize\horizontotal^5\confcountxa}{\batchlength_{\batchindx}}\!\!\bigg)\!.$$
\end{lemma}
By substituting $\accuracy_\batchindx$ in Equation~\ref{eq: regret decomposition with stages}, we will have
\begin{equation}
    \regret\leq \tilde{\cO}\left(\sum_{\batchindx=2}^\batchnumber \batchlength_\batchindx \cdot \left( \sqrt{\frac{\statesize\actionsize\horizontotal^3}{\batchlength_{\batchindx-1}}} + \frac{\statesize^3\actionsize\horizontotal^5\confcountxa}{\batchlength_{\batchindx-1}} \right)
    \right).
\end{equation}
Here, the carefully designed exponential batch schedule becomes essential.
The total regret is a sum over the stages, composed of a non-private term proportional to $\batchlength_{\batchindx}/\sqrt{\batchlength_{\batchindx-1}}$ and a private cost term proportional to $\batchlength_{\batchindx}/\batchlength_{\batchindx-1}$.
By choosing an exponential schedule where $\batchlength_{\batchindx}/\sqrt{\batchlength_{\batchindx-1}}$ forms a geometric series and $\batchlength_{\batchindx}/\batchlength_{\batchindx-1}$ is a constant, we get a great balance between privacy cost and non-private regret, i.e.,~an optimal non-private term and an additive private cost term on regret bound in Theorem~\ref{thm: Regret bound of PBPE}.
Any other schedule, such as a fixed or polynomially increasing batch size, would fail to balance both components, resulting in a suboptimal regret dominated by either the non-private or the private cost.

The proof reduces to proving Lemma~\ref{lemma: total accuracy error}.
We leverage the absorbing MDP $\abtrans$ as the key intermediate bridge, and then decompose this total error into three components, ``Model Bias'' $\vert \valuef^\policy(\reward,\transeasy) - \valuef^\policy(\reward,\abtrans) \vert$,  ``Reward Error'' $\vert\valuef^\policy(\reward,\abtrans) - \valuef^\policy(\rewprieasy^\batchindx,\abtrans)\vert$, and ``Model Variance'' $\vert\valuef^\policy(\rewprieasy^\batchindx,\abtrans) - \valuef^\policy(\rewprieasy^\batchindx,\reftrans)\vert$.
We will now sketch the bounds for ``Model Bias'' and ``Model Variance''. The Reward Error can be bounded as a lower-order term using a similar analysis.

\subsection{``Model Bias'': Difference between $\transeasy$ and $\abtrans$}
\label{subsection: ``Model Bias''}
To analyze the difference between the true MDP with $\transeasy$ and the absorbing MDP with $\abtrans$, we rely on the the properties of the crude transition estimate $\crutrans$, and then reduce the problem to proving it difficult to visit those infrequently visited tuples.

\textbf{Property of $\crutrans$.}
In $\batchindx$-th stage, if the private visitation count $\visitxaxtotalprieasy^{\text{cru},\batchindx}_\horizon(\state,\action,\state')$ for a tuple $(\horizon,\state,\action,\state')$ exceeds $\cO(\confcountxa\horizontotal^2\iota)$, the following approximation result holds with high probability,
$(1-\frac{1}{\horizontotal}) \cdot \crutrans_\horizon(\state'\vert\state,\action) \leq \abtrans_\horizon(\state'\vert\state,\action) \leq (1+\frac{1}{\horizontotal})  \cdot \crutrans_\horizon(\state'\vert\state,\action)$, which can be proven using Bernstein's inequality.
By the construction of $\infreqtuples$, and $\abtrans,\crutrans$, the above equation holds for any $(\horizon,\state,\action,\state')$. 
Consequently, for any $(\horizon,\state,\action)\in[\horizontotal]\times\statespace\times\actionspace,\policy\in\policyset_\batchindx$, we have
$$\frac{1}{4}\valuef^\policy(1_{\horizon,\state,\action},\crutrans) \leq \valuef^\policy(1_{\horizon,\state,\action},\abtrans)\leq 3\valuef^\policy(1_{\horizon,\state,\action},\crutrans).$$

\textbf{Uniform bound on $\vert\valuef^\policy(\reward,\transeasy) - \valuef^\policy(\reward,\abtrans)\vert$.} 
Next, we aim to bound $\sup_{\pi\in\policyset_\batchindx}\sup_{\reward'}\vert\valuef^\policy(\reward',\transeasy) - \valuef^\policy(\reward',\abtrans)\vert$.
This leads to bounding $\sup_{\policy\in\policyset_\batchindx}\prob_\policy[\badtrajs]$, where the event $\badtrajs$ occurs when the trajectory visits any infrequently visited tuples in $\infreqtuples$, which describes the main difference between absorbing MDP and the true MDP.
By the definition of $\infreqtuples$, we can show that these tuples are difficult to visit for any policy in $\policyset_\batchindx$, i.e.,
with high probability, $\sup_{\policy\in\policyset_\batchindx}\prob_\policy[\badtrajs]\leq \tilde{\cO}\rbr{\frac{\statesize^3\actionsize\horizontotal^4\confcountxa}{\batchlength_\batchindx}}$.
Using this observation, we bound the bias with proof in Appendix~\ref{appen: proof of lemma bound of bias term}.
\begin{lemma} \label{lemma: bound of the bias term}
    With high probability, for any policy $\policy\in\policyset_\batchindx$ and reward function $\reward'$, it holds that 
    $$0\leq \valuef^\policy(\reward',\transeasy) - \valuef^\policy(\reward',\abtrans) \leq \widetilde{\cO}\rbr{\frac{\statesize^3\actionsize\horizontotal^5\confcountxa}{\batchlength_\batchindx}}.$$
\end{lemma}

\subsection{``Model Variance'': Difference between $\abtrans$ and $\reftrans$}
This term bounds quantifies the deviation of our refined estimate, $\reftrans$, deviates from the absorbing MDP, $\abtrans$. 
By employing the standard simulation lemma~\cite{dann2017unifying}, this variance term can be approximately bounded by $\widetilde{\cO}\left(\sqrt{\sum_{\horizon,\state,\action} \frac{\valuef^{\policy}(1_{\horizon,\state,\action},\abtrans)}{\visitxatotalref}}\right)$.
These visitation counts are inherently linked to the exploration policy, specifically, $\visitxatotalref \propto \valuef^{\policy^{\text{ref},\batchindx}}(1_{\horizon,\state,\action},\abtrans)$. 
Thus, the key to bounding this term lies in the fine exploration policy $\policy^{\text{ref},\batchindx}$. 
By its design, this policy minimizes the worst-case coverage number over all active policies, defined as $\max_{\policy\in\policyset_\batchindx} \sum_{\horizon,\state,\action}\frac{\valuef^{\policy}(1_{\horizon,\state,\action},\abtrans)}{\valuef^{\policy^{\text{ref},\batchindx}}(1_{\horizon,\state,\action},\abtrans)}$, with a order of $\cO(\statesize\actionsize\horizontotal)$. 
This minimization ensures sufficient exploration of all relevant state-action pairs for policy evaluation, leading to the following guarantee (proof in Appendix~\ref{appen: proof of lemma bound of variance term}).
\begin{lemma}\label{lemma: bound of the variance term}
    With high probability, for any policy $\policy\in\policyset_\batchindx$ and reward function $\reward'$, it holds that 
    $$\vert\valuef^\policy\!(\reward'\!,\!\abtrans) - \valuef^\policy\!(\reward'\!,\!\reftrans)\vert \!\leq \! \widetilde{\cO}\!\bigg(\!\sqrt{\frac{\statesize\actionsize\horizontotal^3}{\batchlength_\batchindx}} + \frac{\statesize^2\actionsize\horizontotal^2\confcountxa}{\batchlength_\batchindx}\!\bigg)\!.$$
\end{lemma}



\section{Privacy Guarantee}
\label{sec: Privacy Guarantees}
In this section, we present a communication protocol served as a shuffle Privatizer, which is adapted from the shuffle binary summation mechanism introduced by \cite{cheu2019distributed}.
We first describe the mechanism satisfying the Assumption~\ref{assp: private counts}, and then show that Algorithm \ref{algo: Private Batch-based Policy Elimination}, when instantiated with this Privatizer, satisfies the $(\pripara,\priconf)$-SDP guarantee.

For a batch of $n$ users, we outline the procedure for a single counter $\visitxaxtotalprieasy_\horizon(\state,\action,\state')$ of a specific $(\horizon, \state,\action,\state^\prime)$. 
This procedure applies to batch data $\batchdata^\horizon$ collected in Crude Exploration or $\batchdata$ from Fine Exploration.
The processing for other counts like $\visitxatotalprieasy_\horizon(\state,\action)$ and rewards $\rewcumprieasy_\horizon(\state,\action)$ is analogous.

Firstly, we allocate privacy budget $\pripara'=\frac{\pripara}{3\horizontotal}$ and failure probability $\priconf'=\frac{\priconf}{\horizontotal\statesize\actionsize}>0$ to each counter, and define a threshold $\tau = \cO(\log(1/\priconf')/\pripara'^2)$ that controls noise addition.
On the local side, each user encodes their data using a randomizer $\randomizer$:
if $n\leq\tau$, user $i$ encodes local count as $z_i = \ind_\horizon\rbr{\state,\action,\state'} + \sum_{j=1}^m y_j$, where $\{y_j\}_{j=1}^m$ are i.i.d. sampled from $\text{Bernoulli}(1/2)$, and $m = \lceil\frac{\tau}{n}\rceil$;
otherwise, the local output is $z_i = \ind_\horizon\rbr{\state,\action,\state'} + y$ where $y$is sampled from $\text{Bernoulli}(\frac{\tau}{2n})$.
Each user then sends its local message $z_i$ to a secure shuffler $\shuffler$, which permutes their messages randomly and forwards them to the analyzer $\analyzer$.

On the analyzer side, it performs aggregation and post-process steps.
It firstly aggregates the shuffled messages into an initial noisy count, denoted as $\visitxaxtotalbineasy_\horizon(\state,\action,\state')$.
If $n\leq\tau$, $\visitxaxtotalbineasy_\horizon(\state,\action,\state') = \sum_{i=1}^n z_i - \lceil\frac{\tau}{n}\rceil \cdot \frac{n}{2}$, otherwise, $\visitxaxtotalbineasy_\horizon(\state,\action,\state') = \sum_{i=1}^n z_i - \frac{\tau}{2}$.
The noisy visitation counts are then post-processed to ensure they satisfy Assumption~\ref{assp: private counts}, and generate the final counts $\visitxatotalprieasy_\horizon(\state,\action,\state')$.
The private cumulative rewards $\rewcumprieasy_\horizon(\state,\action)$ are directly obtained from the aggregation output.

For the post-processing steps for $\visitxaxtotalbineasy_\horizon$,
we will adapt the techniques from \cite{baidifferentially,qiao2023near}.
Firstly, we solve the linear optimization problem
for all $\rbr{\state,\action}$ below.
\begin{equation}
\begin{aligned} 
\label{opt: counter post-processing}
    \min t \,\,\,\,
    \text{s.t.}& \,\, n(\state') \geq 0, \vert n(\state')  -  \visitxaxtotalbineasy_\horizon(\state,\action,\state') \vert  \leq t,  \forall x',\\
    & \vert\sum\nolimits_{\state^\prime\in\statespace} n(\state') - \visitxatotalbineasy_\horizon(\state,\action) \vert \leq \frac{\confcountxa}{4}. 
\end{aligned}
\end{equation}
Let $\visitxaxtotalopteasy_\horizon(\state,\action,\state')$ denote a minimizer of this problem, 
we define $\visitxatotalopteasy_\horizon(\state,\action) = \sum_{\state^\prime\in\statespace} \visitxaxtotalopteasy_\horizon(\state,\action,\state')$.
By adding some term, as done below, the private counts $\visitxatotalprieasy_\horizon(\state,\action)$ never underestimate the respective true counts:
\begin{equation*}
\label{eq: counters - adding term to optimization solution}
\begin{aligned}
    \visitxatotalprieasy_\horizon\!(\state,\!\action)\! =\! \visitxatotalopteasy_\horizon\!(\state,\!\action)\! +\! \frac{\confcountxa}{2}, 
    \visitxaxtotalprieasy_\horizon\!(\state,\!\action,\!\state')\! =\! \visitxaxtotalopteasy_\horizon\!(\state,\!\action,\!\state')\! +\! \frac{\confcountxa}{2\statesize}.
\end{aligned}
\end{equation*}
Suppose $\visitxatotalbineasy_\horizon(\state,\action)$ satisfies  
$\vert\visitxaxtotalbineasy_\horizon(\state,\action,\state') - \visitxaxtotaleasy_\horizon(\state,\action,\state') \vert \leq \frac{\confcountxax}{4}$,
$\vert\visitxatotalbineasy_\horizon(\state,\action) - \visitxatotaleasy_\horizon(\state,\action) \vert \leq \frac{\confcountxax}{4}$, 
for all $\rbr{\horizon,\state,\action,\state^\prime}$, with probability $1-2\delta$. Then, 
the derived $\visitxatotalprieasy_\horizon(\state,\action), \visitxatotalprieasy_\horizon(\state,\action,\state')$ satisfy Assumption~\ref{assp: private counts}.


The formal guarantees of our Shuffle Privatizer, leading to a set of reliable statistics, are summarized in the following lemma with proof in Appendix~\ref{sec: privacy gaurantee appendix}.
\begin{lemma}[Shuffle DP Privatizer]
\label{lemma: Properties of Shuffling-PRIVATIZER}
For any $\pripara \!\in\! (0,1)$, $\priconf\! \in\! (0,1)$, the Shuffling Privatizer satisfies $(\pripara,\priconf)$-SDP and Assumption~\ref{assp: private counts} with $\confcountxa = \tilde{\cO}\rbr{\frac{\horizontotal}{\pripara}}$.
\end{lemma}

As corollaries of Theorem~\ref{thm: Regret bound of PBPE}, we obtain the regret and privacy guarantees for SDP-PE using the shuffle Privatizer.
\begin{corollary}[Regret Bound under SDP]
\label{crl: Regret under SDP}
For any $\pripara\in (0,1),\priconf \in (0,1)$, under the Shuffling Privatizer, SDP-PE satisfies $(\pripara,\priconf)$-SDP.
Furthermore, we obtain 
$\regret \leq \tilde{\cO}\rbr{\sqrt{\statesize\actionsize\horizontotal^3\episodetotal} + \frac{\statesize^3\actionsize\horizontotal^6}{\pripara}}$ with high probability.
\end{corollary}


\section{Numerical Experiments}
\label{sec: experiment}
In this section, we present a series of numerical experiments to validate the effectiveness of our proposed SDP-PE framework (Algorithm~\ref{algo: Private Batch-based Policy Elimination}). 

\subsection{Experimental Setup}
The simulation is conducted on a simplified ``RiverSwim''~\cite{osband2013more} environment, a standard benchmark with challenging exploration requirements, as illustrated in Figure~\ref{fig: RiverSwim}.
The system consists $\statesize=4$ states, $\actionsize=2$ actions ("left", "right"), and episode length $\horizontotal=6$. 
In each episode, the agent starts at $S_1$, receives $\reward=0.005$ at $S_1$, and $\reward=1.0$ for action "right" at $S_4$. 
If the agent chooses the action ``left'', it will always succeed. Otherwise, it may fail.
\begin{figure}[!htbp]
  \centering
  \includegraphics[width=0.48\textwidth]{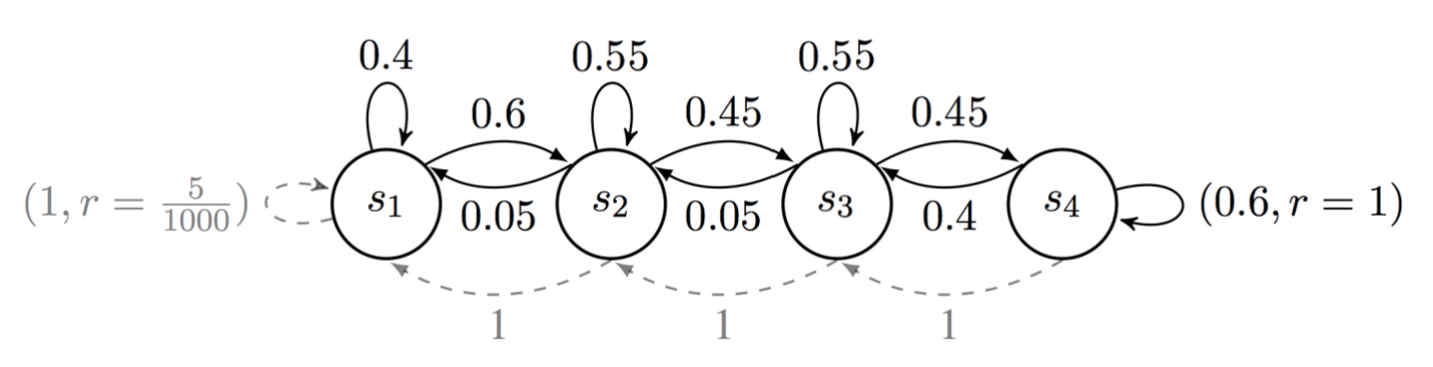}
  \caption{Simplified RiverSwim MDP -- solid and dotted arrows denote the transitions under actions ``right'' and ``left'', respectively.}
  \label{fig: RiverSwim}
\end{figure}

We also experiment in a multi-armed bandit (MAB) environment with $\actionsize=20$ as a generalization of MDP with $\horizontotal=1, \statesize=1$. 
This MAB environment adopts the reward heterogeneity setting considered in~\cite{li2024distributed}, where the global mean reward is $\reward^*(a) \sim \text{Uniform}(0.0, 0.99)$. Each user $k$'s local mean reward is $r_k(a) = r^*(a) + \xi_k(a)$, where $\xi_k(a) \sim \mathcal{N}(0, \sigma^2)$. 
The user observes a random reward from distribution $\text{Bernoulli}(\text{clip}(r_k(a), 0, 1))$.


We evaluate the performance of SDP-PE and compare it against the state-of-the-art private algorithm, private UCB-VI under JDP and LDP constraints~\cite{qiao2023near}, which updates the policy in every episode, as well as their corresponding non-private versions (PE and UCB-VI~\cite{azar2017minimax}).
We conduct 20 independent experiments, each running $\episodetotal = 2\cdot 10^4$ episodes.
The privacy budget $\epsilon$ is specified in each experiment.


\subsection{Results and Analysis}
\subsubsection{Privacy-Utility Trade-off}

\begin{figure*}[t]
    \centering
    \begin{subfigure}{0.3\textwidth}
        \centering
        \includegraphics[width=\linewidth]{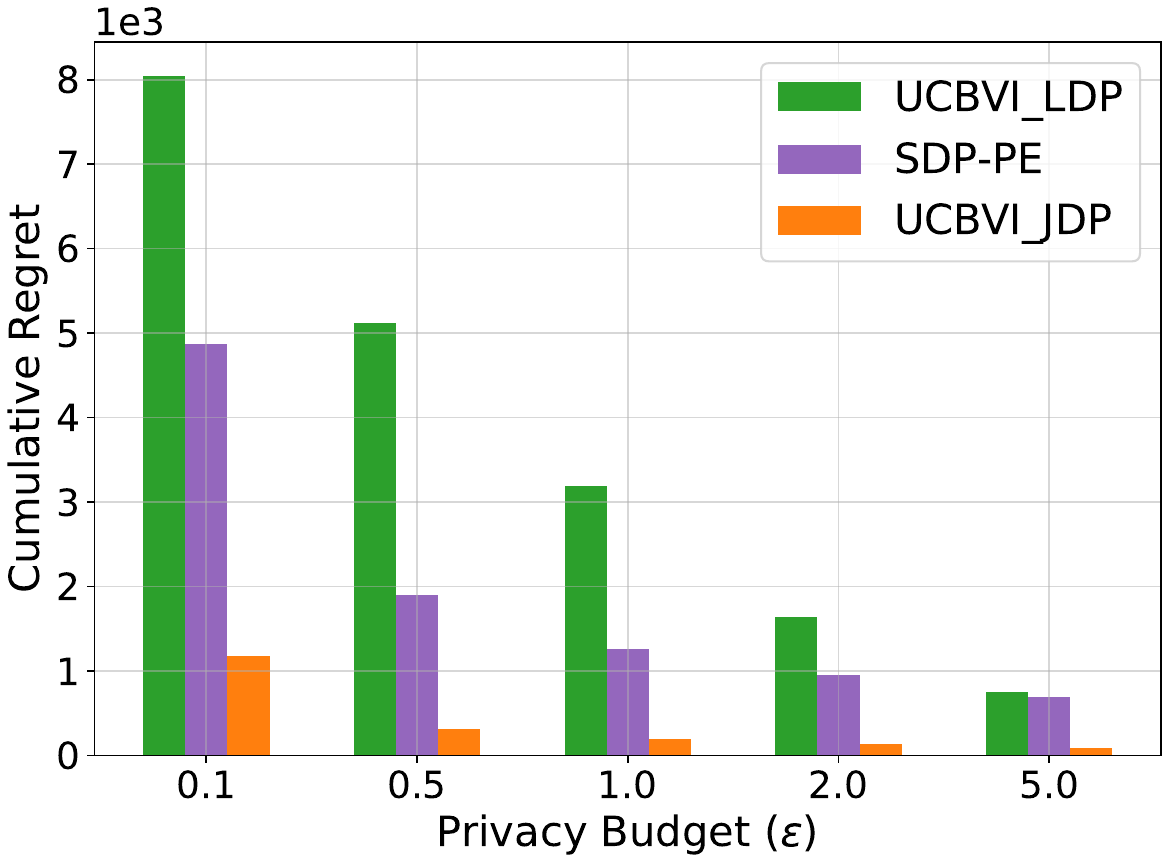} 
        \caption{}
        \label{fig:sub_a}
    \end{subfigure}
    \hfill 
    \begin{subfigure}{0.3\textwidth}
        \centering
        \includegraphics[width=\linewidth]{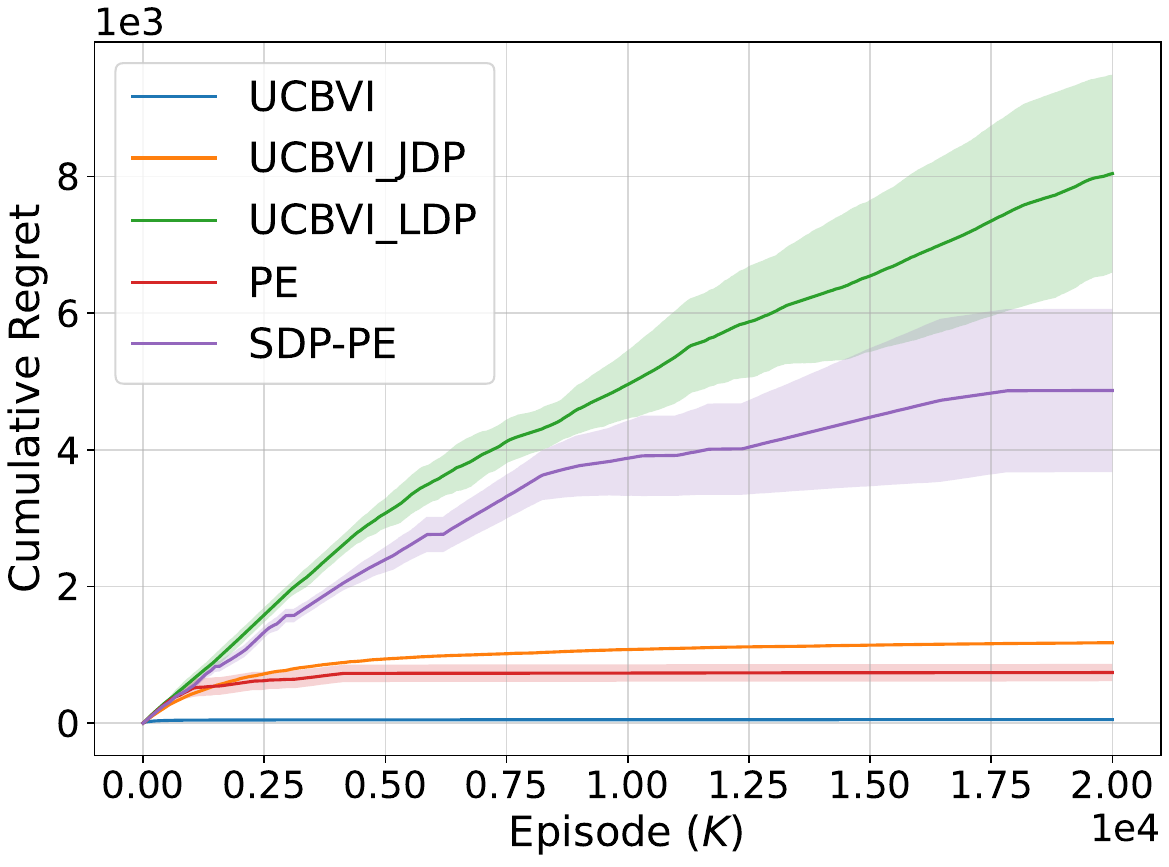}
        \caption{}
        \label{fig:sub_b}
    \end{subfigure}
    \hfill
    \begin{subfigure}{0.3\textwidth}
        \centering
        \includegraphics[width=\linewidth]{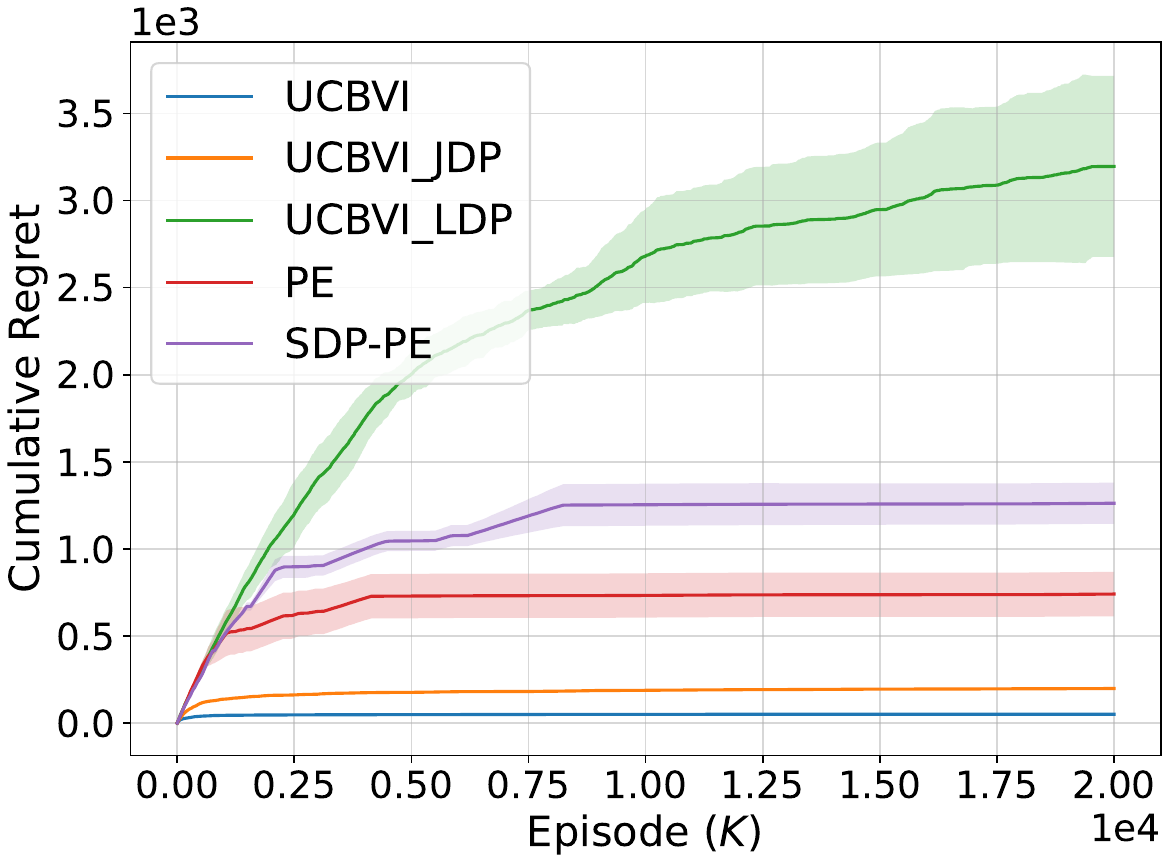}
        \caption{}
        \label{fig:sub_c}
    \end{subfigure}
    \caption{Performance evaluation of SDP-PE under RiverSwim environment. (a) Final cumulative regret vs. the privacy budget $\epsilon$. (b) and (c) display cumulative regret vs. time for comparison against baseline algorithms under $\epsilon = 0.1$ (high privacy regime) and $\epsilon = 1.0$ (low privacy regime), respectively. The shaded area in (b) and (c) indicates the standard deviation.}
    \label{fig: Simulation} 
\end{figure*}

We first evaluate the cumulative regret under different privacy models in the RiverSwim environment.
Figure~\ref{fig: Simulation} illustrates the comparison of cumulative regret for SDP-PE against various baselines in the RiverSwim MDP environment with different privacy budgets.

First, the UCBVI-LDP algorithm exhibits the worst regret growth under all privacy regimes. 
This is because, in RL, LDP requires significant noise to be added at every step, rendering the data almost unusable and making it impractical for real-world network applications.
Second, our SDP-PE algorithm achieves substantially lower regret compared to UCBVI-LDP. 
Its performance closely approaches that of the non-private PE and UCBVI-JDP.
This empirically demonstrates that the shuffle model (SDP) provides a superior balance for distributed networks: it avoids the catastrophic utility loss of LDP without requiring a trusted central server (unlike JDP).

\begin{figure}
    \centering
    \includegraphics[width=0.7\linewidth]{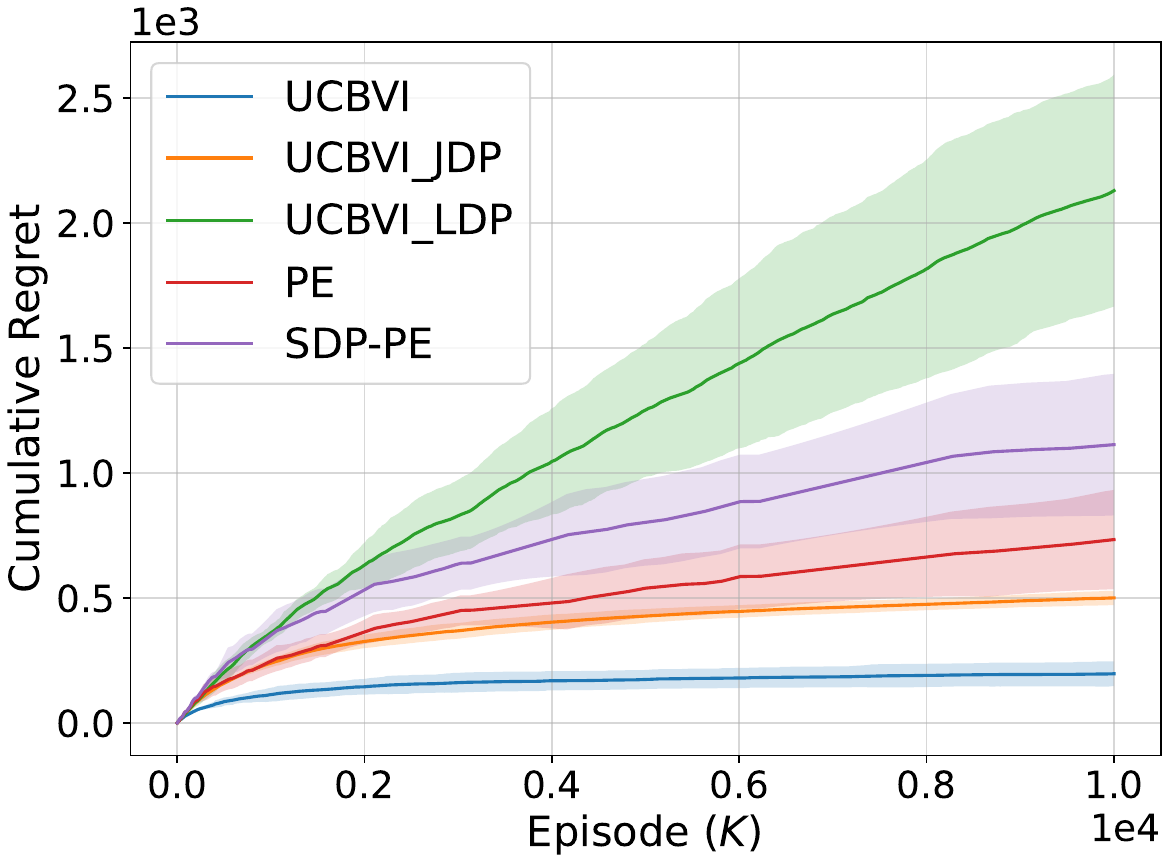}
    \caption{Performance evaluation of SDP-PE under MAB environment. ($\pripara=1.0, \sigma=0.1$)}
    \label{fig:Heterogeneous MAB Regret}
\end{figure}
In addition, we validate our framework's generality to bandits.
As shown in Figure~\ref{fig:Heterogeneous MAB Regret}, our SDP-PE algorithm continues to perform excellently in this MAB setting, again surpassing the LDP baseline and approaching the JDP baseline. 
This confirms our work is a general framework that can be seen as a stateful generalization of the bandit problem, unifying private global optimization for both RL and bandit networks.

Collectively, these results demonstrate that SDP-PE, through the shuffle model of privacy, successfully addresses the privacy bottleneck in online reinforcement learning. 
It offers an effective solution for distributed network environments, combining robust privacy guarantees with superior learning performance, and showcases its excellent generalization capabilities across different online learning scenarios.

\subsubsection{Switching Cost Efficiency}
\begin{figure}
    \centering
    \includegraphics[width=0.75\linewidth]{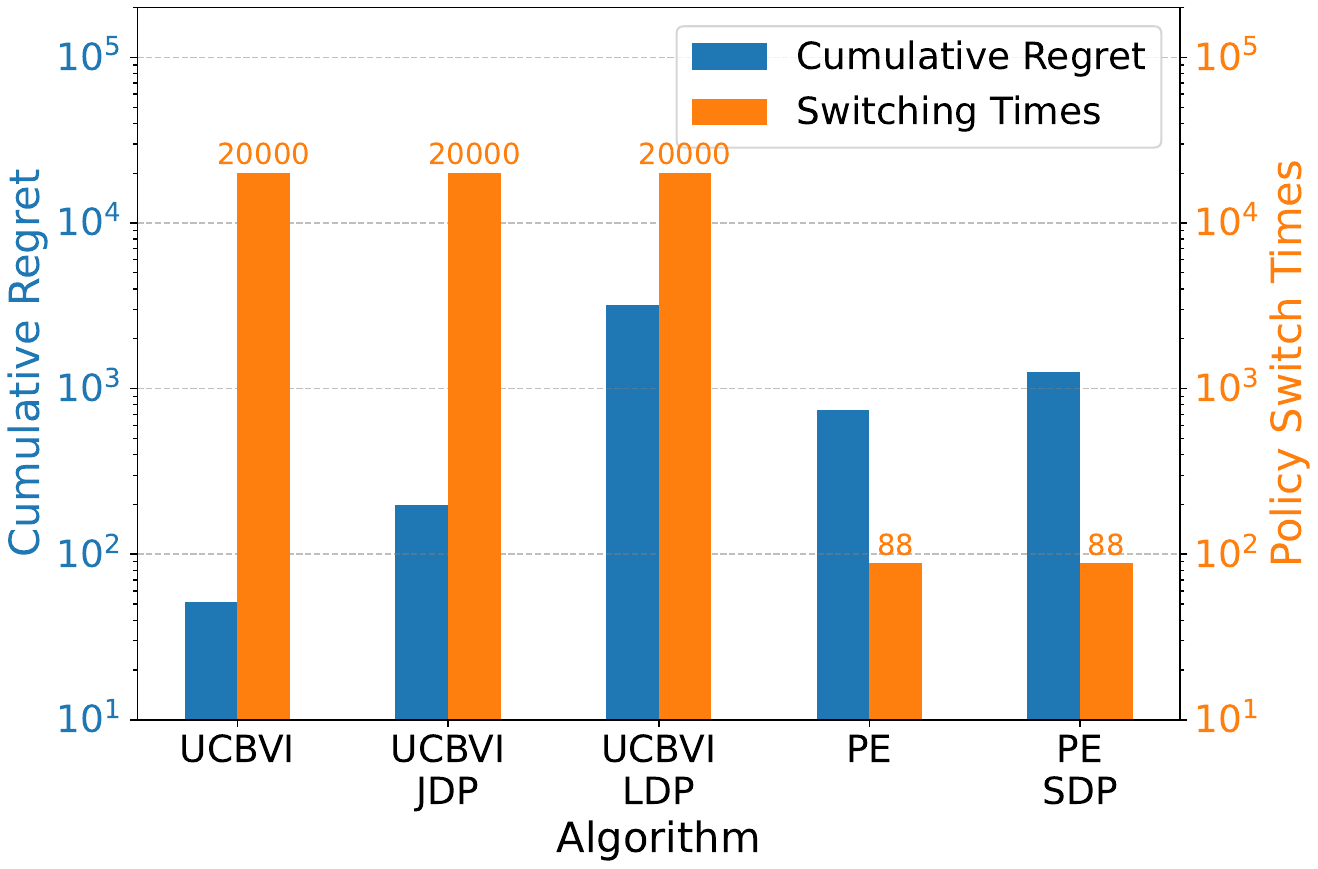}
    \caption{Trade-off between final cumulative regret and policy switch times. ($\episodetotal=20,000, \pripara=1.0$)}
    \label{fig: switch and computation}
\end{figure}
Next, we evaluate SDP-PE's capability in reducing policy deployment cost, a critical engineering challenge for online RL deployments. 
Figure~\ref{fig: switch and computation} compares the total number of policy switches across different algorithms over the entire $K=20,000$ episodes.
Classic UCBVI algorithms (whether UCBVI, UCBVI-LDP, or UCBVI-JDP), due to their "per-interaction" policy update mechanism, incur a policy switching cost as high as $\mathcal{O}(K)$, namely $20,000$ total policy updates.
In a real-world cellular or cloud network, such frequent "reconfiguration" operations are untenable and impractical.
In contrast, our SDP-PE framework (and its non-private PE variant) leverages exponential batching, reducing the network switching cost to $\mathcal{O}(\log K)$, which is only $88$ policy switches in the experiments.
This orders-of-magnitude reduction is attributed to SDP-PE's ingenious utilization of the shuffle model of privacy's inherent batching requirement, combined with its policy elimination mechanism. 
By updating and deploying policies only at the end of each batch, SDP-PE drastically reduces network deployment overhead, making it far more feasible for practical network operations.

These results indicate that SDP-PE is not only superior in terms of privacy but also, by ingeniously leveraging the batching property of SDP, addresses a core engineering challenge of large-scale online RL deployments, significantly enhancing its practical value.

\section{Discussion}
\label{sec: discussion}
Here, we discuss the broader implications of our results in the context of the private learning for networked systems literature.

\textbf{Performance-privacy trade-off.}
Our SDP-PE algorithm establishes a favorable regret-privacy trade-off, with a regret bound whose leading term matches the non-private lower bound, $\Omega(\sqrt{\statesize\actionsize\horizontotal^3\episodetotal})$~\cite{domingues2021episodic}.
Crucially, our algorithm significantly outperforms the optimal LDP result~\cite{qiao2023near} by making the privacy costs' dependency on $1/\pripara$ additive, rather than multiplicative with respect to $\episodetotal$.
The utility matches that of the JDP case, but is achieved within a more resilient, distributed architecture that lacks a single point of failure--a critical feature for robust networked systems.




\textbf{The shuffle model as a resilient privacy architecture.}
The shuffle model serves as an intermediate trust model between the centralized model (JDP) and the local model (LDP) for privacy-preserving learning. 
Firstly, a key theoretical insight is that any RL algorithm satisfying SDP also satisfies JDP via Billboard lemma in~\cite{hsu2014private}.
Therefore, the JDP lower bound of $\widetilde{\Omega}(\sqrt{\statesize\actionsize\horizontotal^3\episodetotal} + \frac{\statesize\actionsize\horizontotal}{\pripara})$ established by \cite{vietri2020private}, also applies to our setting.
It remains an open question whether the lower bound is tight for the SDP setting.
Secondly, while our protocol involves local noise injection, it does not satisfy formal LDP.
Meanwhile, due to the ``privacy blanket'' effect, the privacy preservation becomes less costly as the batch size grows and local noise decreases in the later stages.

\textbf{Policy switching under privacy constraints.}
The SDP-driven batching strategy in our online RL scenario creates a synergy with the engineering need to reduce ``policy switching costs''.
In a non-private setting, achieving a near-optimal regret of $\tilde{\mathcal{O}}(\sqrt{K})$ typically requires at least $\Omega(\log \log K)$ switches~\cite{qiao2022sample}. 
However, the necessity of batching under shuffle privacy to overcome aggregated noise leads our SDP-PE algorithm to an $O(\log K)$ switching frequency. 
This additional switching frequency ($\log K$ vs. $\log \log K$) appears to be necessary to maintain high utility while satisfying strong privacy constraints. 
Thus, it remains an important open question whether it is possible to achieve fewer policy switches (e.g., at the $\Omega(\log \log K)$ level) while satisfying shuffle privacy and simultaneously maintaining near-optimal regret.

\textbf{Limitations and future Work.}
A key limitation of our framework lies in its computational scalability. 
Our SDP-PE, based on policy elimination, relies on explicitly evaluating a potentially large set of active policies and solving constrained optimization problems to construct exploration policies in each stage.
These steps, especially enumerating over all deterministic policies, can lead to exponential time complexity and thus do not scale well to large state-action spaces.
A possible direction is to apply softmax (or other differentiable) representation of the policy space and use gradient-based optimization techniques to find approximate solutions.
Another promising solution could involve designing a constrained policy search module inspired by \cite{zhang2022near}, that implicitly optimize over a constrained policy space rather than maintaining an explicit set.
\section{Related Work}
\label{sec: related work}
Our work lies at the unique intersection of three research domains: (1) learning in distributed networks, (2) privacy in sequential decision-making, and (3) reinforcement learning.

\textbf{Learning paradigms in distributed networks.}
In modern networked systems, learning from distributed entities has become a central challenge. 
Research in this area can be broadly categorized based on its network structure and interaction paradigm.
First, federated learning (FL)~\cite{9060868} is a widely studied paradigm characterized by centralized coordination and decentralized data. 
A central server orchestrates numerous clients, which train models locally and upload updates (rather than raw data) to collaboratively learn a global model. 
FL typically involves clients uploading data or updates in a periodic, synchronous, batch manner.
Secondly, distributed learning~\cite{10.1145/3377454}, in contrast, lacks a central coordinating node. 
Instead, individual nodes in the network learn local or shared models through peer-to-peer communication. 
This paradigm is often more complex, focusing on communication efficiency and the impact of network topology.
Thirdly, multi-agent systems (MAS)~\cite{10439972}, including multi-agent reinforcement learning~\cite{ZENG2022119688}, focus on problems where multiple autonomous agents interact within a shared environment, primarily studying issues related to game theory among agents.
In this paper, we share similarities with FL in its ``central learning agent-distributed nodes'' structure, yet it differs fundamentally. 
The distributed nodes act as independent decision-makers and interact with the environment under policies deployed by the central agent, and provide raw (or perturbed) data to it, rather than training models locally in FL or playing games in MAS.

\textbf{Sequential decision-making in distributed networks.}
Within the aforementioned networked learning paradigms, sequential decision-making problems, particularly bandits and reinforcement learning, have been extensively studied. 
Distributed/federated bandits~\cite{li2024distributed,10552083} apply FL principles to bandit problems, aiming to learn a global decision policy from distributed clients' real-time feedback. 
For instance, \cite{li2024distributed} solves the problem of learning a global model for global regret minimization in a linear bandit setting, where distributed client feedback interacts in a ``parallel batch'' manner.
Distributed reinforcement learning (DRL)~\cite{labbi2025federated} involves distributing RL tasks across multiple computational nodes. 
However, most DRL work focuses on using multiple parallel agents to accelerate the training of a single task, where agents typically share data or models, without addressing user privacy concerns or unique online task stream characteristics.
In contrast to the above synchronous batch models, an ``online Streaming Task Sequence'' model is also studied within bandits and RL framework~\cite{10018417,8022964}.
This model is highly relevant to our setting, where user requests (i.e., episodes) arrive sequentially and dynamically.
Typically, only one user interacts with the system at any given time, and the central agent must learn from this real-time, single-user stream. 

\textbf{Online RL algorithms.}
Regarding online reinforcement learning algorithms, various methods with theoretical guarantees have been extensively investigated. 
These include algorithms designed for tabular MDPs, such as those based on the Optimism in the Face of Uncertainty principle (e.g., UCBVI~\cite{azar2017minimax}), Thompson Sampling (e.g. PSRL~\cite{osband2013more}), Policy Elimination (PE)~\cite{qiao2022sample} and their variants. 
For large state spaces, researchers have also developed linear approximation methods~\cite{jin2020provably} and policy gradient algorithms~\cite{agarwal2021theory} with theoretical guarantees.
This paper primarily focuses on tabular MDP environments, utilizing Policy Elimination algorithms as core building blocks, but these must be fundamentally reshaped under shuffle privacy constraint.

\textbf{Privacy in sequential decision-making.}
In the ``online streaming'' model, the feedback uploaded by each user is highly sensitive, making privacy a core challenge. 
To address this, differential privacy has been widely applied to such sequential decision-making problems.
Firstly, JDP (Centralized Privacy) has been studied in bandits~\cite{shariff2018differentially} and RL~\cite{vietri2020private}, JDP requires a fully trusted central server to deal with the raw data, which is an unrealistic assumption in our decentralized, anonymous crowd setting.
Meanwhile, LDP (Local Privacy) avoids the trust assumption, but applications in RL~\cite{garcelon2021local} and bandits~\cite{ren2020multi} show that the immense noise required to protect every trajectory leads to a severe utility collapse (e.g., a multiplicative $\mathcal{O}(\sqrt{K/\epsilon})$ regret term), rendering it impractical for high-utility applications.
Recently, SDP (Shuffle Privacy) as an intermediate model, has been successfully applied to bandit problems via action elimination ideas~\cite{tenenbaum2021differentially,li2024distributed}, proving it can balance trustlessness and high utility. 
However, to the best of our knowledge, extending the shuffle model to the more complex, online reinforcement learning problem remains an open gap. 
This work is the first to fill this gap, providing a foundation for shuffle-private episodic RL.

\textbf{Batched RL.}
Applying SDP to online RL faces a significant technical challenge, namely, the SDP privacy inherently requires data to be ``batched'' to achieve privacy amplification. 
This constraint is non-trivial for the standard online RL model, as it introduces (a) delayed feedback (the agent must wait for a batch) and (b) aggregated noise (the agent observes a noisy aggregate, not a single trajectory), rendering traditional RL algorithms (like UCBVI~\cite{azar2017minimax}) ineffective.
Concurrently, a separate research area, low-policy-switch RL, also leverages ``batching'' techniques to solve a purely engineering bottleneck (the $\mathcal{O}(K)$ policy switching cost). 
These works reduce switching costs to $\mathcal{O}(\log \log K)$ or $\mathcal{O}(\log K)$~\cite{bai2019provably,zhang2022near,johnson2024sampleefficiency}.
The core technical innovation of this paper is the discovery and exploitation of the synergy between these two seemingly disparate fields. 
We demonstrate that the policy elimination framework, used to solve ``low-switching-cost'', is the ideal vehicle to solve the ``delay and noise'' challenges introduced by SDP, especially when modified with our ``forgetting'' mechanism, and offers a novel paradigm for practical, privacy-preserving online RL in networked systems.
\section{Conclusion and Future Work}
\label{sec: Conclusion and Future Work}
In this paper, we introduce Shuffle Differentially Private Policy Elimination (SDP-PE), the first generic algorithm for reinforcement learning under the shuffle privacy model.
Our method successfully adapts the policy elimination paradigm to the private learning setting by introducing a novel exponential batching schedule paired with a ``forgetting'' mechanism.
When instantiated with our proposed Privatizer, SDP-PE achieves a regret of $\widetilde{\cO}\rbr{\sqrt{\statesize\actionsize\horizontotal^3\episodetotal} + \frac{\statesize^3\actionsize\horizontotal^6}{\pripara}}$, where the leading term matches the information-theoretic lower bound, and privacy costs are additive.
Notably, this result achieves utility comparable to the centralized (JDP) model but under a trustless, distributed and resilient architecture, while significantly outperforming the local (LDP) model.
Furthermore, this work unveils a crucial synergy between privacy-driven batching and the reduction of policy deployment costs, decreasing the switching frequency from $\mathcal{O}(K)$ to $\mathcal{O}(\log K)$. 
This provides a compelling paradigm for data-driven networked systems that demand both privacy and engineering efficiency.
Our numerical experiments further corroborate the theoretical findings and demonstrate the practical advantages.

This work opens several promising avenues for future research. 
Firstly, key directions include improving the computational efficiency of the algorithm and extending its theoretical guarantees to large-scale state-action spaces using function approximation (e.g., deep reinforcement learning).
Secondly, it is worth exploring whether it is possible to achieve fewer policy switches (e.g., at the $\Omega(\log \log K)$ level) while satisfying shuffle privacy and maintaining near-optimal regret.
Finally, it's also important to investigate privacy and robustness in more complex networked environments, such as those with malicious clients or more dynamic network topologies.
\appendices
\section{Supporting Lemmas}
We first provide several supporting lemmas as follows.
\begin{lemma}[Lemma F.4 in \cite{dann2017unifying}] \label{lemma: Lemma F.4 in Dann2017}
    Let $F_i$ for $i=1 \cdots, n$ be a filtration and $X_1, \cdots, X_n$ be a sequence of Bernoulli random variables with $\mathbb{P}\left(X_i=1 \mid F_{i-1}\right)=P_i$ with $P_i$ being $F_{i-1}$-measurable and $X_i$ being $F_i$ measurable. It holds that $\mathbb{P}\left[\exists k\in[n]: \sum_{t=1}^k X_t<\sum_{t=1}^k P_t / 2-W\right] \leq e^{-W}.$
\end{lemma}


\begin{lemma}[Lemma 1 in \cite{zhang2022near}] \label{lemma: distribution coverage lemma}
Let $d>0$ be an integer. Let $\cX\subset(\Delta^d)^m$. Then there exists a distribution $\cD$ over $\cX$, such that 
$\max_{x=\{x_i\}_{i=1}^{dm} \in\cX} \sum_{i=1}^{dm} \frac{x_i}{y_i} = md,$
where $y=\left\{y_i\right\}_{i=1}^{d m}=\mathbb{E}_{x \sim \mathcal{D}}[x]$. Moreover, if $\mathcal{X}$ has a boundary set $\partial \mathcal{X}$ with finite cardinality, we can find an approximation solution for $\mathcal{D}$ in poly$(|\partial \mathcal{X}|)$ time. 
\end{lemma}

\section{Proof of theoretical results}
\subsection{Proof of Lemma~\ref{lemma: bound of the bias term}}
\label{appen: proof of lemma bound of bias term}
This section provides a detailed proof for bounding the ``Model Bias'', the value difference between the true MDP ($\transeasy$) and our absorbing MDP ($\abtrans$). 
The proof proceeds in three main steps. 
First, we show that our private crude model ($\crutrans$) is multiplicatively close to the absorbing model for all well-explored tuples. 
Second, we leverage this to show that for any active policy, the probability of visiting an ``infrequently-visited'' tuple is small. 
Finally, we use this low probability to bound the total difference in value.

We begin by establishing the multiplicative concentration of our private model estimate.
\begin{lemma}\label{lemma-append: multiplicative close of crude transition}
With probability at least $1-4\delta$, for all $\batchindx$, $\forall (\horizon,\state,\action,\state')$ such that $(\horizon,\state,\action,\state')\notin\infreqtuples$, it holds that, 
$(1-\frac{1}{\horizontotal}) \cdot \crutrans_\horizon(\state'\vert\state,\action) \leq \abtrans_\horizon(\state'\vert\state,\action) \leq (1+\frac{1}{\horizontotal})  \cdot \crutrans_\horizon(\state'\vert\state,\action).$
And $\forall(\horizon,\state,\action,\state')\in\infreqtuples$, $\abtrans_\horizon(\state'\vert\state,\action) = \crutrans_\horizon(\state'\vert\state,\action) = 0$.
\end{lemma}
\begin{proof}
As the algorithm design, $\visitxaxtotalcrupri\geq C_1\confcountxa\horizontotal^2\iota$ and $\crutrans_\horizon \rbr{\state' \vert \state, \action} = \frac{\visitxaxtotalcrupri}{\visitxatotalcrupri}$ for all $(\horizon,\state,\action,\state')\notin\infreqtuples$, then
\begin{align*}
    &\abr{\abtrans_\horizon \rbr{\state' \vert \state, \action} - \crutrans_\horizon \rbr{\state' \vert \state, \action} } \\
    &\leq \sqrt{\frac{2\crutrans_\horizon\rbr{\state'\vert\state,\action}\iota}{\visitxatotalcrupri}} + \frac{4\confcountxa\iota}{\visitxatotalcrupri} \\
    &\leq \!\bigg(\sqrt{\frac{2\iota}{\visitxaxtotalcrupri}}\! +\! \frac{4\confcountxa\iota}{\visitxaxtotalcrupri}\!\bigg) \!\cdot\!\crutrans_\horizon \!\rbr{\state' \vert \state, \action} \\
    &\leq \frac{1}{\horizontotal}\cdot \crutrans_\horizon \rbr{\state' \vert \state, \action}.
\end{align*}
The first line is derived from the Bernstein inequality, and the third line is due to the definition of $\infreqtuples$ and the choice of $C_1 = 6$.
The proof is the same for the left side of the inequality.
\end{proof}

Next, we show that the multiplicative closeness of the transition models extends to the value functions computed under them. 
This is a key step that allows us to reason about policy performance using our crude estimate.
\begin{lemma}
\label{lemma: multiplicatively accurate for value functions under absorbing MDP}
Conditioned on the event in Lemma \ref{lemma-append: multiplicative close of crude transition}, for any policy $\pi$ and any $(\horizon, \state, \action) \in[\horizontotal] \times \statespace \times \actionspace$, it holds that
$$\frac{1}{4}\valuef^\policy(1_{\horizon,\state,\action},\crutrans) \leq \valuef^\policy(1_{\horizon,\state,\action},\abtrans)\leq 3\valuef^\policy(1_{\horizon,\state,\action},\crutrans).$$
\end{lemma}
\begin{proof}
    Under the absorbing MDP, for any trajectory $\traj = \{\state_1,\action_1,\dots,\state_\horizon,\action_\horizon\}$ truncated at time step $\horizon$ such that $(\state_\horizon,\action_\horizon)\in\infreqtuples$, we have $\state_{\horizon'}\neq\absorbstate$ for any $\horizon'\leq\horizon-1$.
    It holds that 
    \begin{equation}
    \begin{aligned}
        &\prob[\traj\vert\abtrans,\policy] = \prod_{j=1}^\horizon \policy_j(\action_j\vert\state_j) \times \prod_{j=1}^{\horizon-1}\abtrans_j(\state_{j+1}\vert\state_j,\action_j) \\
        &\leq \rbr{1+\frac{1}{\horizontotal}}^\horizontotal \prod_{j=1}^\horizon \policy_j(\action_j\vert\state_j) \times \prod_{j=1}^{\horizon-1} \crutrans_j(\state_{j+1}\vert\state_j,\action_j) \\
        &\leq 3 \prob[\traj\vert\crutrans,\policy].
    \end{aligned}
    \end{equation}
    The second line is due to Lemma~\ref{lemma-append: multiplicative close of crude transition}. 
    Let $\trajset_{\horizon,\state,\action}$ be the set of truncated trajectories such that $(\state_\horizon,\action_\horizon) = (\state,\action)$. Then
    \begin{equation}
    \begin{aligned}
        &\valuef^\policy(1_{\horizon,\state,\action},\abtrans) = \sum_{\traj\in\trajset_{\horizon,\state,\action}} \prob[\traj\vert\abtrans,\policy] \\
        &\leq 3 \sum_{\traj\in\trajset_{\horizon,\state,\action}}\prob[\traj\vert\crutrans,\policy] = 3 \valuef^\policy(1_{\horizon,\state,\action},\crutrans).
    \end{aligned}
    \end{equation}
    The left side of the lemma can be proven in a similar way, with $\rbr{1-\frac{1}{\horizontotal}}^\horizontotal\geq 1/4$ when $\horizontotal\geq2$.
\end{proof}

With these properties established, we now bound the probability of a trajectory encountering any infrequently visited tuple (the event $\badtrajs$). 
This is the crucial step in bounding the model bias.
\begin{lemma} \label{lemma: bound of bad trajectory event}
    Conditioned on the event in Lemma \ref{lemma-append: multiplicative close of crude transition}, with high probability at least $1-7\delta$, for any stage $\batchindx\in[\batchnumber]$, $\sup_{\policy\in\policyset_\batchindx}\prob_\policy[\badtrajs]\leq \cO\rbr{\frac{\confcountxa\statesize^3\actionsize\horizontotal^4\iota}{\batchlength_\batchindx}}$.
\end{lemma}
\begin{proof}
    Recall that the event $\badtrajs$ is that the trajectory falls some infrequently visited tuple $(\horizon,\state,\action,\state')$ in $\infreqtuples$, 
    We can decompose $\badtrajs$ into a disjoint union of $\{\badtrajs_\horizon\}_{h=1}^H$, where $\badtrajs_\horizon$ is the event that $(\horizon,\state_\horizon,\action_\horizon,\state_{\horizon+1})\in\infreqtuples$ and $(\horizon',\state_{\horizon'},\action_{\horizon'},\state_{\horizon'+1})\notin\infreqtuples$ for all $\horizon'\leq\horizon-1$;
    and $\prob(\badtrajs) = \sum_{\horizon=1}^\horizontotal \prob(\badtrajs_\horizon)$.

    A key first step is to show that the probability of this event is the same under the true MDP ($\transeasy$) and the absorbing MDP ($\abtrans$). 
    This is because the two models only differ after an infrequently-visited tuple is encountered.
    \begin{align*}
        &\prob[\badtrajs_\horizon\vert\transeasy,\policy] = \sum_{\traj\in\badtrajs_\horizon} \prob[\traj\vert\transeasy,\policy] \\ 
        &= \sum_{\traj_{:\horizon+1} \in\badtrajs_\horizon} \prob[(\state_1,\action_1,\dots,\state_\horizon,\action_\horizon)\vert\transeasy,\policy] \transeasy_\horizon(\state_{\horizon+1}\vert\state_\horizon,\action_\horizon) \\
        &= \sum_{\traj_{:\horizon+1} \in\badtrajs_\horizon} \prob[(\state_1,\action_1,\dots,\state_\horizon,\action_\horizon)\vert\abtrans,\policy] \transeasy_\horizon(\state_{\horizon+1}\vert\state_\horizon,\action_\horizon) \\
        &= \sum_{\state\in\statespace,\action} \valuef^\policy(1_{\horizon,\state,\action},\abtrans) \sum_{\state'\in\statespace:(\horizon,\state,\action,\state')\in\infreqtuples} \transeasy_\horizon(\state'\vert\state,\action) \\
        &=\sum_{\state\in\statespace,\action} \valuef^\policy(1_{\horizon,\state,\action},\abtrans) \abtrans_\horizon(\absorbstate\vert\state,\action) \\
        &= \prob[\badtrajs_\horizon\vert\abtrans,\policy].
    \end{align*}
    In the second line, $\traj_{:\horizon+1}$ means the trajectory $\traj$ truncated at $\state_{\horizon+1}$.
    The third line is because for $(\horizon,\state,\action,\state')\notin\infreqtuples$, $\transeasy=\abtrans$.
    The fourth line is because there is a bijection between trajectories that arrive at $(\horizon,\state,\action)$ under absorbing MDP and trajectories in $\badtrajs_\horizon$ that arrive at the same tuple under the original MDP.
    The last equations follow the definition of $\abtrans$ and $\badtrajs_\horizon$.

    The analysis then reduces to bounding $\prob[\badtrajs_\horizon\vert\abtrans,\policy]$. We can express this as the probability of visiting any state $\state$ at step $\horizon$, multiplied by the conditional probability of transitioning to an infrequent tuple from there.
    \begin{equation}
    \begin{aligned}
        &\sup_{\policy\in\policyset_\batchindx} \prob[\badtrajs_\horizon\vert\transeasy,\policy]   = 
        \sup_{\policy\in\policyset_\batchindx} \prob[\badtrajs_\horizon\vert\abtrans,\policy]  \\
        & = \sup_{\policy\in\policyset_\batchindx} \sum_{\state\in\statespace}  \valuef^\policy(1_{\horizon,\state},\abtrans) \max_{\action\in\actionspace} \abtrans_\horizon(\absorbstate\vert\state,\action) \\
        & = \sup_{\policy\in\policyset_\batchindx} \sum_{\state\in\statespace}  \valuef^\policy(1_{\horizon,\state},\abtrans) \max_{\action\in\actionspace} \prob[\infreqtuples\vert(\horizon,\state,\action)], 
    \end{aligned}
    \end{equation}
    where $\prob[\infreqtuples\vert(\horizon,\state,\action)]$ represents the conditional probability of entering $\infreqtuples$ at time step $\horizon$ via $(\state_\horizon,\action_\horizon)=(\state,\action)$.

    Firstly, we deal with the state visiting issue. 
    If the expected number of visits to a state is low, its contribution to the total probability is small.
    Since $\policy^{cru}_{\horizon,\state,\action}=\argmax_{\policy\in\policyset_\batchindx}\valuef^\policy(1_{\horizon,\state,\action},\crutrans)$, and using Lemma~\ref{lemma: multiplicatively accurate for value functions under absorbing MDP}, then 
    \begin{equation}
    \label{eq: inequality of visitation probability functions between randomized policy and the sup policy}
    \begin{aligned}
    \valuef^{\policy^{cru}_{\horizon,\state,\action}}(1_{\horizon,\state,\action},\abtrans) &\geq 
    \frac{1}{4} \valuef^{\policy^{cru}_{\horizon,\state,\action}}(1_{\horizon,\state,\action},\crutrans) \\
    &\geq 
    \frac{1}{12}\sup_{\policy\in\policyset_\batchindx}\valuef^\policy(1_{\horizon,\state,\action},\abtrans).
    \end{aligned}
    \end{equation}
    
    Since $\policy^{cru}_\horizon$ chooses each $\policy^{cru}_{\horizon,\state,\action}$ with probability $\frac{1}{\statesize\actionsize}$ for any $(\state,\action)\in\statespace\times\actionspace$, we have, 
    \begin{equation}
    \begin{aligned}
    \valuef^{\policy^{cru}_\horizon}(1_{\horizon,\state,\action},\abtrans)&
    \geq \frac{1}{12\statesize\actionsize}\sup_{\policy\in\policyset_\batchindx}\valuef^\policy(1_{\horizon,\state},\abtrans).
    \end{aligned}
    \end{equation}
    The equality is because we only select among deterministic policies.
    By Lemma~\ref{lemma: Lemma F.4 in Dann2017} and a union bound, we have with probability $1-\delta$, for any $(\state,\action)\in\statespace\times\actionspace$,
    \begin{equation}
    \begin{aligned}
    \visitxatotaleasy^{\text{cru},\batchindx}_\horizon(\state,\action) &\geq \frac{\batchlength_\batchindx}{\horizontotal} \cdot \frac{\valuef^{\policy^{cru}_\horizon}(1_{\horizon,\state,\action},\transeasy)}{2} - \iota \\
    &\geq \frac{\batchlength_\batchindx\cdot \text{sup}_{\policy\in\policyset_\batchindx} \valuef^\policy(1_{\horizon,\state},\abtrans)}{24\horizontotal\statesize\actionsize}-\iota.
    \end{aligned}
    \end{equation}
    For fixed $(\state,\action)$, if $\frac{\batchlength_\batchindx\cdot \text{sup}_{\policy\in\policyset_\batchindx} \valuef^\policy(1_{\horizon,\state},\abtrans)}{24\horizontotal\statesize\actionsize}\leq 2\iota$, we have that
    \begin{equation}
    \text{sup}_{\policy\in\policyset_\batchindx} \valuef^\policy(1_{\horizon,\state},\abtrans) \leq \frac{48\horizontotal\statesize\actionsize\iota}{\batchlength_\batchindx}.
    \end{equation}
    As a result, we have the bound for any $\policy\in\policyset_\batchindx$,
    \begin{equation}
    \begin{aligned}
    &\valuef^\policy(1_{\horizon,\state},\abtrans) \max_{\action\in\actionspace} \prob[\infreqtuples\vert(\horizon,\state,\action)] \leq \valuef^\policy(1_{\horizon,\state},\abtrans) \\
    &\leq \sup_{\policy\in\policyset_\batchindx} \valuef^\policy(1_{\horizon,\state},\abtrans) \leq \frac{48\horizontotal\statesize\actionsize\iota}{\batchlength_\batchindx}.
    \end{aligned}
    \end{equation}

    Then we handle the opposite case, $\frac{\batchlength_\batchindx\cdot \text{sup}_{\policy\in\policyset_\batchindx} \valuef^\policy(1_{\horizon,\state},\abtrans)}{24\horizontotal\statesize\actionsize} > 2\iota$.
    If the expected number of visits is high, we can use concentration inequalities to show that the probability of transitioning to an infrequent tuple must be very small.
    Denote $N_\horizon^{\text{cru},\batchindx}(\infreqtuples\vert(\horizon,\state,\action))$ as the true times of entering $\infreqtuples$ conditioned at time step $\horizon$, visiting $(\state_\horizon,\action_\horizon)=(\state,\action)$. 
    $\tilde{N}_\horizon^{\text{cru},\batchindx}(\infreqtuples\vert(\horizon,\state,\action))$ denotes the private version.  
    By Lemma~\ref{lemma: Lemma F.4 in Dann2017} and a union bound, with high probability, for any $(\state,\action)\in\statespace\times\actionspace$, 
    \begin{equation}
    \begin{aligned}
    6\confcountxa\horizontotal^2\statesize\iota 
    &\geq \tilde{N}_\horizon^{\text{cru},\batchindx}(\infreqtuples\vert(\horizon,\state,\action)) 
    \geq N_\horizon^{\text{cru},\batchindx}(\infreqtuples\vert(\horizon,\state,\action)) \\
    &\geq    \frac{\visitxatotaleasy^{\text{cru},\batchindx}_\horizon(\state,\action) \prob[\infreqtuples\vert(\horizon,\state,\action)]}{2}-\iota. 
    \end{aligned}
    \end{equation}
    The first inequality is due to the definition of $\infreqtuples$, and the second inequality is because of our optimistic property of the private counter.
    Thus, this results in 
    \begin{equation}
        \prob[\infreqtuples\vert(\horizon,\state,\action)] \leq \frac{672\confcountxa\statesize^2\actionsize\horizontotal^3\iota}{\batchlength_\batchindx\sup_{\policy\in\policyset_\batchindx}\valuef^\policy(1_{\horizon,\state},\abtrans)}.
    \end{equation}
    In this case, we have, 
    \begin{equation}
    \begin{aligned}
    &\valuef^\policy(1_{\horizon,\state},\abtrans) \max_{\action\in\actionspace} \prob[\infreqtuples\vert(\horizon,\state,\action)] \\
    &\leq \valuef^\policy(1_{\horizon,\state},\abtrans) \cdot \frac{672\confcountxa\statesize^2\actionsize\horizontotal^3\iota}{\batchlength_\batchindx\sup_{\policy\in\policyset_\batchindx}\valuef^\policy(1_{\horizon,\state},\abtrans)} \\
    &\leq \frac{672\confcountxa\statesize^2\actionsize\horizontotal^3\iota}{\batchlength_\batchindx}.
    \end{aligned}
    \end{equation}

    Combining both cases, 
    \begin{equation}
    \begin{aligned}
        &\sup_{\policy\in\policyset_\batchindx} \prob[\badtrajs_\horizon\vert\transeasy,\policy]   = 
        \sup_{\policy\in\policyset_\batchindx} \sum_{\state\in\statespace}  \valuef^\policy(1_{\horizon,\state},\abtrans) \max_{\action\in\actionspace} \prob[\infreqtuples\vert(\horizon,\state,\action)] \\
        &\leq \sup_{\policy\in\policyset_\batchindx} \sum_{\state\in\statespace} \max\rbr{\frac{48\horizontotal\statesize\actionsize\iota}{\batchlength_\batchindx}, \frac{672\confcountxa\statesize^2\actionsize\horizontotal^3\iota}{\batchlength_\batchindx}} \\
        &= \frac{672\confcountxa\statesize^3\actionsize\horizontotal^3\iota}{\batchlength_\batchindx},
    \end{aligned}
    \end{equation}
Summing over $\horizontotal$ layers, we have with high probability,
\begin{equation}
\begin{aligned}
\sup_{\policy\in\policyset_\batchindx}\prob_\policy[\badtrajs]\leq \cO\rbr{\frac{\confcountxa\statesize^3\actionsize\horizontotal^4\iota}{\batchlength_\batchindx}}.
\end{aligned}
\end{equation}
\end{proof}

Finally, we get the upper bound of the Model Bias term. 
    For any reward function $\reward'$, the non-negative inequality is obvious due to the ``absorbing'' MDP definition.
    For the right-hand side, we have that for any policy $\policy\in\policyset_\batchindx$,
    \begin{align*}
        \valuef^\policy(\reward',\transeasy) &= \sum_{\traj\notin\badtrajs}\reward'(\traj)\prob[\traj\vert\transeasy,\policy] + \sum_{\traj\in\badtrajs} \reward'(\traj)\prob[\traj\vert\transeasy,\policy] \\
        &= \sum_{\traj\notin\badtrajs}\reward'(\traj)\prob[\traj\vert\abtrans,\policy] + \sum_{\traj\in\badtrajs} \reward'(\traj)\prob[\traj\vert\transeasy,\policy] \\
        &\leq \valuef^{\policy}(\reward',\abtrans) + \sum_{\traj\in\badtrajs} \horizontotal \prob[\traj\vert\transeasy,\policy] \\
        &\leq \valuef^{\policy}(\reward',\abtrans) + \horizontotal\prob[\badtrajs\vert\transeasy,\policy] \\
        &\leq \valuef^{\policy}(\reward',\abtrans) + \frac{672\confcountxa\statesize^3\actionsize\horizontotal^5\iota}{\batchlength_\batchindx},
    \end{align*}
    where the inequality follows Lemma~\ref{lemma: bound of bad trajectory event}.

\subsection{Proof of Lemma~\ref{lemma: bound of the variance term}}
\label{appen: proof of lemma bound of variance term}
This section provides the proof for Lemma~\ref{lemma: bound of the variance term}, which bounds the difference between the value functions under the absorbing MDP ($\abtrans$) and the refined private estimate ($\reftrans$).
The proof first establishes that the fine exploration phase gathers sufficient data. 
It then introduces two key supporting lemmas regarding the uniform coverage of our fine exploration policy and the resulting visitation counts. 
These are then used in the main proof to bound the variance term.

A key prerequisite for this proof is that the fine exploration phase gathers sufficient visiting times for tuples not in $\infreqtuples$.
We establish this by leveraging the auxiliary policy $\policy_0$.
By its construction in the crude exploration phase and the multiplicative Chernoff bound, we can show that for any tuple not in $\infreqtuples$, running the auxiliary policy $\policy_0$ for $\batchlength_\batchindx$ episodes ensures its visitation count is sufficiently large (i.e., more than $ C_1\confcountxa\horizontotal^2\iota$). 
This ensures that the favorable properties of our model estimates, such as those established for the crude estimate in Lemma~\ref{lemma-append: multiplicative close of crude transition} and Lemma~\ref{lemma: multiplicatively accurate for value functions under absorbing MDP}, also hold for the refined estimate, $\reftrans$.

We then establish two supporting lemmas: a uniform coverage guarantee for our fine exploration policy $\policy^{\text{ref},\batchindx}$ and a high-probability bound on the resulting visitation counts.

\begin{lemma}\label{lemma-appen: coverage dependence lemma}
    Conditioned on the event in Lemma~\ref{lemma-append: multiplicative close of crude transition}, the policy $\policy^{\text{ref},\batchindx}$ in Algorithm~\ref{algo: Fine Exploration} satisfies that 
    $\max_{\mu\in\policyset} \sum_{\horizon,\state,\action} \frac{\valuef^{\mu}(1_{\horizon,\state,\action},\abtrans)}{\valuef^{\policy^{\text{ref},\batchindx}}(1_{\horizon,\state,\action},\abtrans)} \leq 12 \statesize\actionsize\horizontotal. $
\end{lemma}
\begin{proof}
    For some policy $\mu$, it holds that $L.H.S \!\leq \!\sum_{\horizon,\state,\action} \!\! \frac{3 \valuef^{\mu}(1_{\horizon,\state,\action},\crutrans)}{\frac{1}{4}\valuef^{\policy^{\text{ref},\batchindx}}(1_{\horizon\!,\state\!,\action},\crutrans)} \!\leq\!\! 12\statesize\actionsize\horizontotal,$
where the first inequality is due to the property of $\crutrans$ in Lemma~\ref{lemma: multiplicatively accurate for value functions under absorbing MDP}, and the last inequality is by plugging $\mathcal{X}=\{\{\valuef^{\mu}(1_{\horizon,\cdot,\cdot},\crutrans)\}_{\horizon=1}^\horizontotal \mid \mu \in \policyset\}, d=\statesize\actionsize$ and $m=\horizontotal$ into Lemma~\ref{lemma: distribution coverage lemma}.
\end{proof}

Second, we provide a high-probability lower bound on the visitation counts collected by the fine exploration policy for all relevant state-action pairs.
\begin{lemma}\label{lemma: bounds of visitation counts in fine exploration}
With probability at least $1-4\delta$, for all $(\horizon,\state,\action)$, at least one of the following statements holds:
(1) For all $\state'\in\statespace$, $(\horizon,\state,\action,\state')\in\infreqtuples$, i.e., $\abtrans_\horizon(\absorbstate\vert\state,\action)=1$ is fixed and known;
(2) $\visitxatotalref\geq\frac{1}{2}\batchlength_\batchindx\valuef^{\policy^{\text{ref},\batchindx}}(1_{\horizon,\state,\action},\abtrans)-\iota$.
\end{lemma}
\begin{proof}
Statement (1) is true by the definition.
For statement (2), consider any $(\horizon,\state,\action)$ tuple, the expectation number of visits to this tuple when running $\policy^{\text{ref},\batchindx}$ for $\batchlength_\batchindx$ episodes is $\batchlength_\batchindx\valuef^{\policy^{\text{ref},\batchindx}}(1_{\horizon,\state,\action},\transeasy)\geq \batchlength_\batchindx\valuef^{\policy^{\text{ref},\batchindx}}(1_{\horizon,\state,\action},\abtrans)$ by definition. 
Thus, according to Lemma~\ref{lemma: Lemma F.4 in Dann2017}, with probability $1-\frac{\delta}{\horizontotal\statesize\actionsize\episodetotal}$, $\visitxatotalref\geq$ $\frac{1}{2}\batchlength_\batchindx\valuef^{\policy^{\text{ref},\batchindx}}(1_{\horizon,\state,\action},\transeasy)-\iota\geq \frac{1}{2}\batchlength_\batchindx\valuef^{\policy^{\text{ref},\batchindx}}(1_{\horizon,\state,\action},\abtrans)-\iota$.
A union bound over all such tuples ensures that statement (2) holds with overall high probability.
\end{proof}
Then we are ready to prove Lemma~\ref{lemma: bound of the variance term}.
For simplification, we define $\{f_h(\cdot)\}_{\horizon=1}^\horizontotal$ to be the value function of one policy $\mu$ under $\abtrans$ and $\reward'$. 
Then it holds that, for any policy $\mu\in\policyset_\batchindx$,
\begin{align*}
& \vert\valuef^\mu\!(\reward'\!,\!\abtrans) - \valuef^\mu\!(\reward'\!,\!\reftrans)\vert \\
\leq&  \sum_{\horizon,\state,\action} \abr{(\reftrans_\horizon - \abtrans_\horizon)f_{\horizon+1}(\state,\action)}\valuef^\mu(1_{\horizon,\state,\action},\reftrans)  \\
\leq& \!\sum_{\horizon,\state,\action}\!\! \Bigg(\frac{\visitxatotalref}{\visitxatotalrefpri}\bigg(\!\sqrt{\frac{2\text{Var}_{\abtrans(\cdot\vert\state,\action)}f_{\horizon+1}}{\visitxatotalref}} \!+\! \frac{2\horizontotal\iota}{3\visitxatotalref}\!\bigg) \\
&+\frac{2\statesize\horizontotal\confcountxa}{\visitxatotalrefpri} \Bigg) \cdot 4 \valuef^\mu(1_{\horizon,\state,\action},\abtrans) \\
\leq& 10\statesize\horizontotal\confcountxa\iota\underbrace{\sum_{\horizon,\state,\action}\frac{\valuef^\mu(1_{\horizon,\state,\action},\abtrans)}{\visitxatotalref}}_{(i)} + 4\sqrt{2\underbrace{\sum_{\horizon,\state,\action}\frac{\valuef^\mu(1_{\horizon,\state,\action},\abtrans)}{\visitxatotalref}}_{(i)}} \\
&\cdot\sqrt{\underbrace{\sum_{\horizon,\state,\action}\valuef^\mu(1_{\horizon,\state,\action},\abtrans)\text{Var}_{\abtrans(\cdot\vert\state,\action)}f_{\horizon+1}}_{(ii)}}\\
\leq & \cO\bigg(\frac{\statesize^2\actionsize\horizontotal^2\confcountxa\iota}{\batchlength_\batchindx}+ \sqrt{\frac{\statesize\actionsize\horizontotal^3\iota}{\batchlength_\batchindx}}\bigg).
\end{align*}    
The first inequality follows the simulation lemma.
The second inequality holds by the concentration of $\reftrans$ by Bernstein inequality, Assumption~\ref{assp: private counts}, and Lemma~\ref{lemma: multiplicatively accurate for value functions under absorbing MDP}.
The third inequality is derived via Cauchy-Schwarz inequality, and the property of private counts.
The last inequality is obtained by bounding terms $(i)$ and $(ii)$, as detailed below.

The upper bound for $(i)$ is as follows.
\begin{align*}
(i)& \leq  \sum_{\horizon,\state,\action}\frac{\valuef^\mu(1_{\horizon,\state,\action},\abtrans)}{\max\{1, \frac{1}{2}\batchlength_\batchindx\valuef^{\policy^{\text{ref},\batchindx}}(1_{\horizon,\state,\action},\abtrans)-\iota\}}\\
\leq & \sum_{\horizon,\state,\action}\valuef^\mu(1_{\horizon,\state,\action},\abtrans) \big(\II\{ \batchlength_\batchindx\valuef^{\policy^{\text{ref},\batchindx}}(1_{\horizon,\state,\action},\abtrans)< 4\iota\} \\
&+\! \min\!\big\{1,\!\frac{4}{\batchlength_\batchindx\!\valuef^{\policy^{\text{ref},\batchindx}}\!(1_{\horizon,\state,\action},\!\abtrans)}\big\}\II\{ \batchlength_\batchindx\valuef^{\policy^{\text{ref},\batchindx}}\!(1_{\horizon,\state,\action},\!\abtrans)\!\geq\! 4\iota\} \big) \\
\leq & \sum_{\horizon,\state,\action}\!\frac{\valuef^\mu(1_{\horizon,\state,\action},\abtrans)}{\valuef^{\policy^{\text{ref},\batchindx}}\!(1_{\horizon,\state,\action},\!\abtrans)}\!\cdot\!\frac{4\iota}{\batchlength_\batchindx} + \sum_{\horizon,\state,\action}\!\frac{\valuef^\mu(1_{\horizon,\state,\action},\abtrans)}{\valuef^{\policy^{\text{ref},\batchindx}}\!(1_{\horizon,\state,\action},\abtrans)}\!\cdot\!\frac{4}{\batchlength_\batchindx} \\
\leq& \frac{48\horizontotal\statesize\actionsize}{\batchlength_\batchindx} +\frac{48\horizontotal\statesize\actionsize\iota}{\batchlength_\batchindx}.
\end{align*}
The first line is due to Lemma~\ref{lemma: bounds of visitation counts in fine exploration}, and the second inequality uses the standard decomposition; the last follows Lemma~\ref{lemma-appen: coverage dependence lemma}.

The upper bound for $(ii)$ is shown below: $(ii) = \sum_{\horizon,\state,\action}\valuef^\mu(1_{\horizon,\state,\action},\abtrans)\text{Var}_{\abtrans(\cdot\vert\state,\action)}f_{\horizon+1}(\cdot) \leq \horizontotal^2$, where the inequality results from a recursive application of Law of Total Variance.


\subsection{Proof of Theorem~\ref{thm: Regret bound of PBPE}}
\label{appendix: proof of regret-theorem}
The proof of Theorem~\ref{thm: Regret bound of PBPE} consists of two main parts: bounding the policy switching cost and bounding the cumulative regret.
We first give the proof of the upper bound on the policy switching times.
The total number of episodes is $\episodetotal$. 
The learning process is partitioned into $\batchnumber$ stages, indexed by $\batchindx=1, 2, \dots, \batchnumber$. 
In each stage $\batchindx$, we utilize $3\batchlength_\batchindx$ episodes, where $\batchlength_\batchindx$ episodes are for crude exploration (Algorithm~\ref{algo: Crude Exploration}) and $2\batchlength_\batchindx$ episodes are for fine exploration (Algorithm~\ref{algo: Fine Exploration}).
The batch lengths are chosen as $\batchlength_\batchindx = 2^{\batchindx}$, and then the total batch number is $\batchnumber = \min\{j:3\sum_{\batchindx=1}^j 2^\batchindx \geq \episodetotal \}$. 
Take $j = \lceil \log_2(\episodetotal/6 + 1)\rceil$, we have $3\sum_{\batchindx=1}^j 2^\batchindx \geq 3 (2^{\log_2(\episodetotal/6 + 1)} -2)\geq \episodetotal$. 
Thus,
\begin{equation}
    \batchnumber \leq \log_2(\episodetotal/6 + 1) + 1 = \cO(\log \episodetotal).
\end{equation}
In each stage $\batchindx$, policies are switched $\horizontotal$ times in crude exploration to explore each layer of MDP, and twice in fine exploration. 
Therefore, the number of policy switching, $N_{switch}$, is bounded by:
\begin{equation}
    N_{switch} \leq (\horizontotal+2)\batchnumber \leq \cO(\horizontotal\log \episodetotal).
\end{equation}

We now proceed with bounding the cumulative regret.
We first prove the error bound of the value estimate of any policy. 
\begin{lemma}[Restate of Lemma~\ref{lemma: total accuracy error}]
With probability $1-9\delta$, it holds that with some constant $C$, for any stage $\batchindx$ and $\policy\in\policyset_{\batchindx}$, 
$$\vert\valuef^\policy\!(\reward,\!\transeasy) - \valuef^\policy\!(\rewprieasy^\batchindx\!,\reftrans)\vert \!\leq\!C\Big(\sqrt{\frac{\statesize\actionsize\horizontotal^3\iota}{\batchlength_{\batchindx}}} + \frac{\statesize^3\actionsize\horizontotal^5\confcountxa\iota}{\batchlength_{\batchindx}}\Big)\!.$$
\end{lemma}
\begin{proof}
In each stage, we first construct the absorbing MDP $\abtrans$, then construct an empirical estimate of $\abtrans$, which is $\reftrans$ as an intermediate term.
By applying the triangle inequality, for any $\batchindx$ and any $\policy\in\policyset^\batchindx$, we decompose this error into three terms
$\vert\valuef^\policy(\reward,\transeasy) - \valuef^\policy(\rewprieasy^\batchindx,\reftrans)\vert \leq \vert \valuef^\policy(\reward,\transeasy) - \valuef^\policy(\reward,\abtrans) \vert + \vert\valuef^\policy(\reward,\abtrans) - \valuef^\policy(\rewprieasy^\batchindx,\abtrans)\vert + \vert\valuef^\policy(\rewprieasy^\batchindx,\abtrans) - \valuef^\policy(\rewprieasy^\batchindx,\reftrans)\vert $.
Combining the bounds from Lemma~\ref{lemma: bound of the bias term} and Lemma~\ref{lemma: bound of the variance term}, and taking a union bound over all batches $\batchindx$, we ensure that with probability at least $1-9\delta$, the value estimation error bound holds for any $\batchindx\leq\batchnumber$ and any $\policy\in\policyset^\batchindx$.
\end{proof}

We leverage two key lemmas to prove the regret bound.
\begin{lemma}\label{lemma: optimal policy is in the policy set}
    Conditioned on the same high probability event of Lemma \ref{lemma: total accuracy error}, the optimal policy $\policy^*$ will never be eliminated, i.e., $\policy^*\in\policyset_\batchindx$ for $\batchindx=1,2,3,\dots,\batchnumber$.
\end{lemma}
\begin{proof}
    The proof proceeds by induction, since $\policyset_1$ contains all the deterministic policies, $\policy^\star \in \policy_1$.
    Assume $\policy^\star \in \policyset_\batchindx$, then we show that 
    \begin{align*}
        &\sup_{\policy'\in\policyset_\batchindx} \valuef^{\policy'}(\rewprieasy^{\text{ref},\batchindx},\reftrans) - \valuef^{\policy^\star}(\rewprieasy^{\text{ref},\batchindx},\reftrans) \\
        =&  \valuef^{\policy^\batchindx}(\rewprieasy^{\text{ref},\batchindx},\reftrans) - \valuef^{\policy^\star}(\rewprieasy^{\text{ref},\batchindx},\reftrans) \\
        \leq & 
        \abr{\valuef^{\policy^\batchindx}(\rewprieasy^{\text{ref},\batchindx},\reftrans) \!-\! \valuef^{\policy^\batchindx}(\reward,\transeasy)} 
        \!+\! \valuef^{\policy^\batchindx}(\reward,\transeasy) \!-\! \valuef^{\policy^\star}(\reward,\transeasy) \\
        &+ \abr{\valuef^{\policy^\star}(\reward,\transeasy) - \valuef^{\policy^\star}(\rewprieasy^{\text{ref},\batchindx},\reftrans)} \\
        \leq & 2C\rbr{\sqrt{\frac{\statesize\actionsize\horizontotal^3\iota}{\batchlength_{\batchindx}}} + \frac{\statesize^3\actionsize\horizontotal^5\confcountxa\iota}{\batchlength_{\batchindx}}}.
    \end{align*}
Thus, the estimated value of the optimal policy $\policy^\star$ will not be sufficiently lower than the estimated best policy from $\policyset_\batchindx$.
According to the elimination step in Algorithm \ref{algo: Private Batch-based Policy Elimination}, we will have $\policy^\star \in \policy_{\batchindx+1}$, which means the optimal policy will never be eliminated.
Thus, the optimal policy $\policy^\star$ will never be eliminated throughout the entire learning process.
\end{proof}

\begin{lemma}\label{lemma: regret in each batch}
    Conditioned on the same high probability event of Lemma \ref{lemma: optimal policy is in the policy set}, for any remaining policy $\policy\in\policyset_\batchindx$, we have 
    $$\abr{\valuef^{\policy^*}(\reward,\transeasy) \!- \!\valuef^\policy(\reward,\transeasy)} \!\leq 4C\rbr{\sqrt{\frac{\statesize\actionsize\horizontotal^3\iota}{\batchlength_{\batchindx}}} \!+\! \frac{\statesize^3\actionsize\horizontotal^5\confcountxa\iota}{\batchlength_{\batchindx}}}.$$
\end{lemma}
\begin{proof}
    Since the optimal policy $\policy^\star$ is never eliminated, we know that for $\policy\in\policyset_{\batchindx+1}$, its estimated value must be ``close'' to that of $\policy^\star$. 
    According to the policy elimination rule, it must hold that:
    \begin{align*}
         &\valuef^{\policy^\star}(\rewprieasy^{\text{ref},\batchindx},\reftrans) -  \valuef^{\policy}(\rewprieasy^{\text{ref},\batchindx},\reftrans) \\
         \leq& \sup_{\policy'\in\policyset_\batchindx} \valuef^{\policy'}(\rewprieasy^{\text{ref},\batchindx},\reftrans) -  \valuef^{\policy}(\rewprieasy^{\text{ref},\batchindx},\reftrans) \\
         \leq& 2C\rbr{\sqrt{\frac{\statesize\actionsize\horizontotal^3\iota}{\batchlength_{\batchindx}}} + \frac{\statesize^3\actionsize\horizontotal^5\confcountxa\iota}{\batchlength_{\batchindx}}}.
    \end{align*}
    Now, we bound the true sub-optimality gap for any surviving policy $\policy$ by applying the triangle inequality:
    \begin{align*}
        &\abr{\valuef^{\policy^*}(\reward,\transeasy) - \valuef^\policy(\reward,\transeasy)}  \\
        \leq& \abr{\valuef^{\policy^*}(\reward,\transeasy) - \valuef^{\policy^*}(\rewprieasy^{\text{ref},\batchindx},\reftrans)} \\
         &+ \abr{\valuef^{\policy^*}(\rewprieasy^{\text{ref},\batchindx},\reftrans) - \valuef^\policy(\rewprieasy^{\text{ref},\batchindx},\reftrans)} \\
         &+ \abr{\valuef^{\policy}(\rewprieasy^{\text{ref},\batchindx},\reftrans) - \valuef^\policy(\reward,\transeasy)} \\
        \leq& 4C \rbr{\sqrt{\frac{\statesize\actionsize\horizontotal^3\iota}{\batchlength_{\batchindx}}} + \frac{\statesize^3\actionsize\horizontotal^5\confcountxa\iota}{\batchlength_{\batchindx}}}.
    \end{align*}
\end{proof}
Finally, we aggregate the regret over all stages to obtain the total cumulative regret for Theorem~\ref{thm: Regret bound of PBPE}.
    The regret for the first stage is at most $3\horizontotal\batchlength_1 = 6\horizontotal$, because the policies are chosen without much information yet.
    For stage $\batchindx\geq 2$, the policies we use (any policy $\policy\in\policyset_\batchindx$) are at most $4C\rbr{\sqrt{\frac{\statesize\actionsize\horizontotal^3\iota}{\batchlength_{\batchindx-1}}} + \frac{\statesize^3\actionsize\horizontotal^5\confcountxa\iota}{\batchlength_{\batchindx-1}}}$ sub-optimal, as established by Lemma \ref{lemma: regret in each batch} (note the $\batchlength_{\batchindx-1}$ as it's based on the previous stage's elimination).
    Therefore, the regret for the $\batchindx$-th stage ($3\batchlength_\batchindx$ episodes) is at most $12C\batchlength_\batchindx\rbr{\sqrt{\frac{\statesize\actionsize\horizontotal^3\iota}{\batchlength_{\batchindx-1}}} + \frac{\statesize^3\actionsize\horizontotal^5\confcountxa\iota}{\batchlength_{\batchindx-1}}}$.
    The total regret is then bounded by summing over all stages:
    \begin{align*}
        &\regret \!\leq\! 6\horizontotal \!+\! \sum_{\batchindx=2}^{\batchnumber} 12C\batchlength_\batchindx\rbr{\sqrt{\frac{\statesize\actionsize\horizontotal^3\iota}{\batchlength_{\batchindx-1}}}\! + \!\frac{\statesize^3\actionsize\horizontotal^5\confcountxa\iota}{\batchlength_{\batchindx-1}}} \\
        &\leq 6\horizontotal + \cO\rbr{\sqrt{\statesize\actionsize\horizontotal^3\episodetotal\iota}} + \cO\rbr{\statesize^3\actionsize\horizontotal^5\confcountxa\iota\log\episodetotal} \\
        &\leq \tilde{\cO}\rbr{\sqrt{\statesize\actionsize\horizontotal^3\episodetotal} + \statesize^3\actionsize\horizontotal^5\confcountxa}.
    \end{align*}

\subsection{Proof of Lemma~\ref{lemma: Properties of Shuffling-PRIVATIZER}}
\label{sec: privacy gaurantee appendix}
Consider a batch dataset $\batchdata$ of $n$ users, we construct private counts $\visitxatotalprieasy_\horizon(\state,\action,\state'), \visitxaxtotalprieasy_\horizon(\state,\action), \rewcumprieasy_\horizon(\state,\action)$ for all $(\horizon,\state,\action,\state')$ by using the proposed Privatizer in Section~\ref{sec: Privacy Guarantees}. 
We prove the two claims of Lemma~\ref{lemma: Properties of Shuffling-PRIVATIZER} separately: the privacy guarantee and the utility guarantee. 

\textbf{Privacy guarantee.}
The proof follows three steps:
the adapted shuffle binary summation mechanism is $(\pripara',\priconf')$-SDP for one counter, 
then the Privatizer is $(\pripara,\priconf)$-SDP for one batch of users, 
and then the entire mechanism for all users is $\batchnumber$-batch $(\pripara,\priconf)$-SDP.

    Firstly, we analyze the shuffle binary mechanism with parameter $(\pripara',\priconf')$.
    Consider two neighboring input $D=(0,d_1,\dots,d_n)$ and $D'=(1,d_1,\dots,d_n)$, we define the random variable $Q$ to be the sum of all the random bits (i.e., $y$ or $y_1,\dots,y_m$) over all users in $D$, and similarly define $Q'$ with respect to $D'$.
    We first claim that the output of the shuffler $(\shuffler\circ\randomizer^n)(D) = Q + \sum_{i=1}^n d_i$ is $(\pripara',\priconf')$-DP. 
    We denote this random mechanism as $M^*(\cdot)$.

    Since $Q$ is binomial in both regimes, by Chernoff bounds, for any $\priconf'>0$, it holds that $\prob(\abr{Q-\expect[Q]} > \sqrt{3\expect[Q]\log(2/\priconf')})\leq \priconf'$.
    Therefore, define confidence interval as $I_c = (\expect[Q] - \sqrt{3\expect[Q]\log(2/\priconf')}, \expect[Q] + \sqrt{3\expect[Q]\log(2/\priconf')})$, and we get $\prob(Q\notin I_c)\leq \priconf'$.
    A similar result also applies for $Q'$.

    When $n\leq\tau$, $Q\sim\text{Binomial}(\lceil\frac{\tau}{n}\rceil \cdot n, 1/2)$, and $\expect[Q] = \lceil\frac{\tau}{n}\rceil\cdot
    \frac{n}{2}$.
    For any $q \in I_c$, and $\pripara'<1$ it holds
    \begin{align*}
        &\frac{\prob(Q=q)}{\prob(Q'=q-1)} = \frac{2\expect[Q] - q + 1}{q} \\
        & \leq \frac{\expect[Q] + \sqrt{3\expect[Q]\log(2/\priconf')} + 1}{\expect[Q] - \sqrt{3\expect[Q]\log(2/\priconf')}} \\
        &\leq \frac{\tau/2 + \sqrt{\tau/2\cdot 3\log(2/\priconf')} + 1}{\tau/2 - \sqrt{\tau/2\cdot 3\log(2/\priconf')}} \\
        &= \frac{1 + \sqrt{6\log(2/\priconf')/\tau} + 2/\tau}{1 - \sqrt{6\log(2/\priconf')/\tau}} \\
        &= \frac{1+\pripara'/4 + 2/\tau}{1-\pripara'/4} \leq \frac{1+\pripara'/4 + \pripara'/4}{1-\pripara'/4} = \frac{1+\pripara'/2}{1-\pripara'/4} \leq e^{\pripara'}. 
    \end{align*}

    When $n>\tau$, $Q\sim\text{Binomial}(n, \frac{\tau}{2n})$, and $\expect[Q] = \frac{\tau}{2}$.
    For any $q \in I_c$, and $\pripara'<1$, it holds
    \begin{align*}
    &\frac{P(Q=q)}{P\left(Q'=q-1\right)}  =\frac{n-q+1}{q} \cdot \frac{\frac{\tau}{2 n}}{1-\frac{\tau}{2 n}} \\
    &\leq \frac{n-\tau / 2+\sqrt{\frac{3}{2} \tau \log (2/\priconf')}+1}{\tau / 2-\sqrt{\frac{3}{2} \tau \log (2/\priconf')}} \cdot \frac{\frac{\tau}{2 n}}{1-\frac{\tau}{2 n}} \\
    & =\frac{n-\tau / 2+\sqrt{\frac{3}{2} \tau \log (2/\priconf')}+1}{n-\tau / 2} \cdot \frac{\tau / 2}{\tau / 2-\sqrt{\frac{3}{2} \tau \log (2/\priconf')}} \\
    & =\left(1+\frac{\sqrt{\frac{3}{2} \tau \log (2/\priconf')}+1}{n-\tau / 2}\right) \cdot \frac{1}{1-\sqrt{6 \log (2/\priconf') / \tau}} \\
    & \leq\left(1+\sqrt{6 \log (2/\priconf') / \tau}+2 / \tau\right) \cdot \frac{1}{1-\sqrt{6 \log (2/\priconf') / \tau}} \\
    & =\frac{1+\pripara' / 4+2 / \tau}{1-\pripara' / 4} \leq \frac{1+\pripara' / 4+\pripara' / 4}{1-\pripara' / 4} \leq e^{\pripara'}. 
    \end{align*}
    Therefore, we can conclude that in both regimes of $n$, $\forall q\in I_c, \frac{\prob(Q=q)}{\prob(Q'=q-1)} \leq e^{\pripara'}$, and similarly $\frac{\prob(Q=q)}{\prob(Q'=q-1)} \geq e^{-\pripara'}$.
    
    Let $\Gamma = \sum_{j=1}^n d_j$ be the true sum of $D$, and also be the true sum of $D'$ minus one. 
    Therefore, for any subset $E\subseteq \NN$, we have
    \begin{equation}
    \begin{aligned}
        &\prob[M^*(D)\in E] 
        \leq \prob[M^*(D)\in E \wedge Q\in I_c] + \prob[Q\notin I_c] \\
        &\leq \sum_{s\in E} \prob[M^*(D)=s \wedge Q\in I_c] + \priconf \\
        &= \sum_{s\in E} \prob[Q=s-\Gamma \wedge s-\Gamma\in I_c] + \priconf' \\
        &\leq e^{\pripara'} \cdot \sum_{s\in E} \prob[Q'=s-\Gamma-1 \wedge s-\Gamma\in I_c] + \priconf' \\
        &\leq e^{\pripara'} \cdot \sum_{s\in E} \prob[M^*(D')=s] + \priconf \\
        &= e^{\pripara'} \cdot \prob[M^*(D')\in E] + \priconf'. 
    \end{aligned} 
    \end{equation}
    A dual argument can also be obtained in a similar way, and then we can conclude $M^*$ (i.e., $(\shuffler\circ\randomizer^n)$) is $(\pripara',\priconf')$-DP.

    Next, we address the composition for one batch of users.
    For each $(\horizon,\state,\action)$, we invoke the shuffle mechanism once. 
    Using the standard composition theorem~\cite{dwork2014algorithmic} over all $(\horizon,\state,\action)$, and fact that one user difference only results in at most $\horizontotal$ differences in these counters, similar to Lemma 34 of \cite{hsu2014private}, the release of $\{\visitxatotalbineasy_\horizon(\state,\action)\}_{\horizon,\state,\action}$ satisfies $(\frac{\pripara}{3},\priconf)$-SDP due to the choice of $(\pripara',\priconf')$, i.e., $\pripara'=\frac{\pripara}{3\horizontotal}$ and $\priconf'=\frac{\priconf}{\horizontotal\statesize\actionsize}$.
    Similarly, the release of $\{\visitxaxtotalbineasy_\horizon(\state,\action)\}_{\horizon,\state,\action,\state'}$ and $\{\rewcumprieasy_\horizon(\state,\action)\}_{\horizon,\state,\action}$ also satisfy $(\frac{\pripara}{3},\priconf)$-SDP.
    Then, by post-processing theorem and standard composition theorem in \cite{dwork2014algorithmic} the release of all private counts $\{\visitxatotalprieasy_\horizon(\state,\action)\}_{\horizon,\state,\action}, \{\visitxaxtotalprieasy_\horizon(\state,\action,\state')\}_{\horizon,\state,\action,\state'},  \{\rewcumprieasy_\horizon(\state,\action)\}_{\horizon,\state,\action}$ satisfy $(\pripara,\priconf)$-SDP.
    Therefore, by post-processing theorem, for all $\batchnumber$ batches users, our Privatizer satisfies that $\batchnumber$-batch $(\pripara,\priconf)$-SDP.

\textbf{Utility guarantee.}
    Now, we provide the utility guarantee.
    We first analyze the noise on the true summation.
    In both regimes of $n$, the output of the mechanism is of the form $Z = \sum_{i=1}^n d_i + Q - \expect[Q]$, and the mechanism is unbiased since $\expect[Z -\sum_{i=1}^n d_i] = \expect[Q - \expect[Q]] = 0$.
    The error is $Q - \expect[Q]$, which is independent of the input $D$ and depends only on the problem parameters.
    Meanwhile, since in both cases $Q$ is binomial and $\expect[Q]\leq \tau$.
    By Chernoff bounds, we get for any $t>0$, $\prob(Q-\expect[Q]\leq t)\leq \exp(\frac{-t^2}{3\expect[Q]}) \leq \exp(\frac{-t^2}{3\tau})$ and $\prob(Q-\expect[Q]\geq -t) \leq \exp(\frac{-t^2}{3\tau})$.
    This is the equivalent definition of a sub-Gaussian variable with variance being $\tau = O(\frac{\log(1/\priconf)}{{\pripara'}^2})$.
    Therefore, the error confidence directly follows the concentration of sub-Gaussian variables, i.e., for any $t>0$, we have $\abr{Z-\sum_{i=1}^n d_i}>O\rbr{\frac{\sqrt{\log(1/\priconf') \log(1/t)}}{\pripara'}}$ with high probability $1-t$.

    Applying the confidence bound to each counter, we can establish the utility bound that with probability $1-3\delta$ for all $(\horizon,\state,\action,\state')$,
    \begin{equation}
    \begin{aligned}
    &\abr{\visitxatotalbineasy_\horizon(\state,\action) - \visitxatotaleasy_\horizon(\state,\action)} \leq \tilde{\cO}(\frac{\horizontotal}{\pripara}), \\ &\abr{\visitxaxtotalbineasy_\horizon(\state,\action,\state') - \visitxaxtotaleasy_\horizon(\state,\action,\state')} \leq \tilde{\cO}(\frac{\horizontotal}{\pripara}), \\ 
    &\abr{\rewcumprieasy_\horizon(\state,\action) - \rewcumeasy_\horizon(\state,\action)}\leq \tilde{\cO}(\frac{\horizontotal}{\pripara}).
    \end{aligned}
    \end{equation}
    Referring to the post-processing procedures in Section~\ref{sec: Privacy Guarantees}, the shuffle Privatizer satisfies Assumption~\ref{assp: private counts} with $\confcountxa = \tilde{\cO}\rbr{\frac{\horizontotal}{\pripara}}$, and $\visitxatotaleasy_\horizon(\state,\action) \leq \visitxatotalprieasy_\horizon(\state,\action) \leq \visitxatotaleasy_\horizon(\state,\action) + \confcountxa$.
    Furthermore, with the constraints of the optimization problem, we observe $\visitxatotalbineasy_\horizon(\state,\action) = \sum_{\state'}\visitxaxtotalbineasy_\horizon(\state,\action,\state')$, which also implies that $\visitxatotalprieasy_\horizon(\state,\action) = \sum_{\state'}\visitxaxtotalprieasy_\horizon(\state,\action,\state')$.

\bibliographystyle{IEEEtran}
\bibliography{main}

\end{document}